\documentclass[sigconf]{aamas} 

\usepackage{balance} 

\usepackage{amsfonts}       
\usepackage{amsmath,bm}
\usepackage{csquotes}
\usepackage{subcaption} 

\usepackage[export]{adjustbox}

\usepackage[ruled, linesnumbered, noend]{algorithm2e}

\SetCommentSty{mycommfont}
\SetKwComment{Comment}{$\triangleright$\ }{}

\definecolor{oceanboatblue}{rgb}{0.0, 0.47, 0.75}

\newcommand{\vect}[1]{\boldsymbol{\mathbf{#1}}}
\newcommand{\Ex}{\mathbb{E}}
\newcommand{\Aspace}{\mathcal{A}}
\newcommand{\Sspace}{\mathcal{S}}
\newcommand{\Wspace}{\mathcal{W}}
\newcommand{\w}{\vect{w}}
\DeclareMathOperator*{\argmax}{arg\,max}

\newcommand{\pigpi}{\pi^{\text{GPI}}}

\newcommand{\ccs}{\mathrm{CCS}}
\newcommand{\gpi}{\mathrm{GPI}}
\newcommand{\model}{p}
\newcommand{\buffer}{\mathcal{B}}

\newcommand{\Vset}{\mathcal{V}}
\newcommand{\Wset}{\mathcal{M}}
\newcommand{\Wsupport}{\mathcal{M}}

\setcopyright{ifaamas}
\acmConference[AAMAS '23]{Proc.\@ of the 22nd International Conference
on Autonomous Agents and Multiagent Systems (AAMAS 2023)}{May 29 -- June 2, 2023}
{London, United Kingdom}{A.~Ricci, W.~Yeoh, N.~Agmon, B.~An (eds.)}
\copyrightyear{2023}
\acmYear{2023}
\acmDOI{}
\acmPrice{}
\acmISBN{}

\acmSubmissionID{1114}

\title[Sample-Efficient MORL via GPI Prioritization]{Sample-Efficient Multi-Objective Learning via\\ Generalized Policy Improvement Prioritization}

\author{Lucas N. Alegre}
\additionalaffiliation{
\institution{Artificial Intelligence Lab - Vrije Universiteit Brussel}
\city{Brussels}
\country{Belgium}
}
\author{Ana L. C. Bazzan}
\affiliation{
  \institution{Institute of Informatics - Federal University of Rio Grande do Sul}
  \city{Porto Alegre - RS}
  \country{Brazil}}
\email{{lnalegre,bazzan}@inf.ufrgs.br}
  
\author{Diederik M. Roijers}
\author{Ann Nowé}
\affiliation{
\institution{Artificial Intelligence Lab - Vrije Universiteit Brussel}
\city{Brussels}
\country{Belgium}}
\email{{diederik.roijers,ann.nowe}@vub.be}

\author{Bruno C. da Silva}
\affiliation{
  \institution{University of Massachusetts}
  \city{Amherst - MA}
  \country{USA}}
\email{bsilva@cs.umass.edu}

\begin{abstract}
Multi-objective reinforcement learning (MORL) algorithms tackle sequential decision problems where agents may have different \textit{preferences} over (possibly conflicting) reward functions.
Such algorithms often learn a set of policies (each optimized for a particular agent preference) that can later be used to solve problems with novel preferences.
We introduce a novel algorithm that uses Generalized Policy Improvement (GPI) to define principled, formally-derived prioritization schemes that improve sample-efficient learning. They implement active-learning strategies by which the agent can \textit{(i)} identify the most promising preferences/objectives to train on at each moment, to more rapidly solve a given MORL problem; and \textit{(ii)} identify which previous experiences are most relevant when learning a policy for a particular agent preference, via a novel Dyna-style MORL method.
We prove our algorithm is guaranteed to always converge to an optimal solution in a finite number of steps, or an $\epsilon$-optimal solution (for a bounded $\epsilon$) if the agent is limited and can only identify possibly sub-optimal policies.
We also prove that our method monotonically improves the quality of its partial solutions while learning. Finally, we introduce a bound that characterizes the maximum utility loss (with respect to the optimal solution) incurred by the partial solutions computed by our method throughout learning.
We empirically show that our method outperforms state-of-the-art MORL algorithms in challenging multi-objective tasks, both with discrete and continuous state and action spaces.
\end{abstract}


\newcommand{\BibTeX}{\rm B\kern-.05em{\sc i\kern-.025em b}\kern-.08em\TeX}

\begin{document}


\pagestyle{fancy}
\fancyhead{}


\maketitle

 
\section{Introduction}

Reinforcement learning (RL) algorithms \cite{Sutton&Barto2018} have achieved remarkable successes in a wide range of complex tasks~(e.g., \cite{Silver+2017,Bellemare+2020,Wurman+2022}).
In RL, tasks are typically modeled via a single scalar reward function, which encodes the agent's objective. In many real-life settings, by contrast, agents are tasked with optimizing behaviors that trade-off between multiple---possibly conflicting---objectives, each of which is modeled via a reward function. As an example, consider a robot that needs to trade off between locomotion speed, battery usage, and accuracy in reaching a goal location. 
Multi-Objective RL (MORL)~\cite{Hayes+2022} algorithms tackle such a challenge.
In MORL settings, a \textit{utility function} exists that combines (into a single scalar value) the agent's preferences toward optimizing each of its objectives. The goal of MORL algorithms is to rapidly identify policies that maximize the utility function, given any preferences an agent may have.

Standard RL algorithms are often sample-inefficient since they may require the agent to interact with its environment a large number of times~\cite{Sutton&Barto2018}. This is often infeasible, e.g., in applications where interactions are costly or risky. The sample efficiency problem is exacerbated in MORL algorithms, since they often have to learn not a single policy, but a \textit{set} of policies---each designed to optimize a particular trade-off between the agent's preferences \cite{Yang+2019}.

In this paper, we introduce a novel MORL algorithm, with important theoretical guarantees, that improves sample efficiency via two novel prioritization techniques. Such principled techniques allow the agent to \textbf{\textit{(i)}} identify the most promising preferences/objectives to train on at each moment, to more rapidly solve a MORL problem; and \textbf{\textit{(ii)}} identify which previous experiences are most relevant when learning a policy for a particular agent preference, via a novel Dyna-style MORL method.
We show formal analyses and proofs regarding our algorithm's convergence properties, as well as the maximum utility loss (performance) incurred by the transient solutions computed by our method throughout the learning process.

To design a method with these properties, we exploit an important concept in the RL literature: Generalized Policy Improvement (GPI)~\cite{Barreto+2020}. GPI is a policy transfer technique that---given a set of policies---constructs a novel policy guaranteed to be at least as good as any of the available ones.
First, recall that if the utility function of a MORL problem is a linear combination of the agent's objectives, optimal solutions are sets of policies known as \textit{convex coverage sets} (CCS)~\cite{Roijers+2013}. Given a CCS, agents can \textit{directly} identify the optimal solution to any novel preferences.
MORL algorithms that learn a CCS (e.g., \cite{Alegre+2022, Mossalam+2016}) may be sample inefficient \textbf{\textit{(i)}} due to the heuristics they use to determine which preferences to train on, at any given moment during the construction of a CCS; and \textbf{\textit{(ii)}} because they can only improve a CCS after optimal (or near-optimal) policies are identified---which may require a large number of samples. We address the first issue via a novel GPI-based prioritization technique for selecting which preferences to train on. It is based on a lower bound on performance improvements that are formally guaranteed to be achievable, and more accurately and reliably identifies the most relevant preferences to train on (when constructing a CCS) compared to existing heuristics. 
To address the second issue, we show that our method is an anytime algorithm that incrementally/monotonically improves the quality of its CCS, even if given \textit{intermediate} (possibly sub-optimal) policies for different preferences. This improves sample efficiency: our method identifies intermediate CCSs with formally bounded maximum utility loss even if there are constraints on the number of times the agent can interact with its environment.
Our algorithm is guaranteed to always converge to an optimal solution in a finite number of steps, or an $\epsilon$-optimal solution (for a bounded $\epsilon$) if the agent is limited and can only identify possibly sub-optimal policies.
We prove that it monotonically improves the quality of its CCS and formally bound the maximum utility loss (with respect to an optimal solution) incurred by any of its partial solutions.

A complementary approach for increasing sample efficiency is to use a model-based approach to learn policies for different preferences. Once a model is learned, it can be used to identify policies for \textit{any} preferences, thus minimizing the required number of interactions with the environment. We introduce a novel Dyna-style MORL algorithm---the first model-based MORL technique capable of dealing with continuous state spaces. Dyna algorithms are based on generating simulated experiences to more rapidly update a value function or policy. An important question, however, is \textit{which} artificial experiences should be generated to accelerate learning. We introduce a new, principled GPI-based prioritization technique for identifying which experiences are most relevant to rapidly learn the optimal policy for novel preferences. 

In summary, we introduce two novel GPI-based prioritization techniques for use in MORL settings to improve sample efficiency. We formally show that our algorithm is supported by important theorems characterizing its convergence properties and performance bounds. 
We empirically show that our method outperforms state-of-the-art MORL algorithms in challenging multi-objective tasks, both with discrete and continuous state and action spaces.

\section{Background}

In this section, we discuss important definitions (and corresponding notation) associated with MORL, GPI, and model-based RL.

\subsection{Multi-Objective Reinforcement Learning}
\label{sec:morl_intro}

The multi-objective RL setting (MORL) is used to model problems where an agent needs to optimize possibly conflicting objectives, each modeled via a separate reward function. MORL problems are modeled as \textit{Multi-objective Markov decision processes} (MOMDP), which differ from regular MDPs in that the reward function is vector-valued.
A MOMDP is defined as a tuple $M \equiv (\mathcal{S},\mathcal{A},p,\vect{r},\mu,\gamma)$, where $\mathcal{S}$ is a state space,  $\mathcal{A}$ is an action space, $p(\cdot|s,a)$ is the distribution over next states given state $s$ and action $a$, $\vect{r} : \mathcal{S} {\times} \mathcal{A} {\times} \mathcal{S} {\mapsto} \mathbb{R}^m$ is a multi-objective reward function with $m$ objectives, $\mu$ is an initial state distribution, and $\gamma \in [0,1)$ is a discounting factor.
Let $S_t$, $A_t$, and $\vect{R}_t = \vect{r}(S_t,A_t,S_{t+1})$ be random variables corresponding to the state, action, and vector reward, respectively, at time step $t$. 
A policy $\pi : \Sspace \mapsto \Aspace$ is a function mapping states to actions.
The \textit{multi-objective action-value function} of a policy $\pi$ for a given state-action pair $(s,a)$ is defined as:
\begin{eqnarray}
\label{eq:mo-q-function}
     \vect{q}^\pi(s,a) \equiv \Ex_\pi \left[ \sum_{i=0}^{\infty} \gamma^i \vect{R}_{t+i}  \ | \ S_t=s, A_t=a \right],
\end{eqnarray}
where $\vect{q}^\pi(s,a)$ is an $m$-dimensional vector whose $i$-th entry is the expect return of $\pi$ under the $i$-th objective, and $\Ex_{\pi}[\cdot]$ denotes expectation over trajectories induced by $\pi$.
Let $\vect{v}^\pi {\in} \mathbb{R}^m$ be the \textit{multi-objective value vector} of $\pi$ under the initial state distribution $\mu$:
\begin{equation}
\label{eq:vectorvalue}
    \vect{v}^\pi \equiv \Ex_{S_0 \sim \mu} \left[ \vect{q}^\pi(S_0, \pi(S_0)) \right],
\end{equation}
where $v^{\pi}_{i}$ is the value of $\pi$ under the $i$-th objective.
For succinctness, we henceforth refer to $\vect{v}^\pi$ as the \textit{value vector} of policy $\pi$. 
A \textit{Pareto frontier} is a set of nondominated multi-objective value functions $\vect{v}^\pi$: 
$
    \mathcal{F} \equiv \{\vect{v}^\pi \ |\ \nexists \ \pi' \mathrm{ s.t. } \ \vect{v}^{\pi'} \succ_p \vect{v}^{\pi} \}, 
$
where $\succ_p$ is the \textit{Pareto dominance relation} $\vect{v}^\pi \succ_p \vect{v}^{\pi'} \iff (\forall i : v^\pi_i \geq v^{\pi'}_{i}) \land  (\exists i : v^{\pi}_i > v^{\pi'}_{i})$.
In general, the optimal solution to a MOMDP is a set of all policies $\pi$ such that $\vect{v}^\pi$ is in the Pareto frontier.

Let a \textit{user utility function} (or scalarization function) $u : \mathbb{R}^m \mapsto \mathbb{R}$ be a mapping from the multi-objective value of policy $\pi$, $\vect{v}^\pi$, to a scalar. 
Utility functions often linearly combine the value of a policy under each of the $m$ objectives using a set of weights $\w$:
$
   u(\vect{v}^\pi, \w)
    = v_{\w}^{\pi}
    = \vect{v}^\pi\cdot \w,
$
\noindent where each element of $\w \in \mathbb{R}^m$ specifies the relative importance of each objective.
The space of weight vectors, $\Wspace$, is an $m$-dimensional simplex: $\textstyle\sum_i w_i = 1, w_i \geq 0, i=1,...,m$.
For any given constant $\w$ (i.e., a particular way of weighting objectives), the original MOMDP collapses into an MDP with reward function $r_{\w}(s,a,s') = \vect{r}(s,a,s') \cdot \w$.
Given a linear utility function $u$, we can define a \textit{convex coverage set} (CCS)  \cite{Roijers+2013} as a finite convex subset of $\mathcal{F}$, such that there exists a policy in the set that is optimal with respect to any linear preference $\w$. In other words, the CCS is the set of nondominated multi-objective value functions $\vect{v}^\pi$ where the dominance relation is now defined over scalarized values:
\begin{equation}
\label{eq:ccs}
    \ccs \equiv \{\vect{v}^\pi \in \mathcal{F} \ | \ \exists \w\ \text{s.t.}\ \forall \vect{v}^{\pi'} \in \mathcal{F}, \vect{v}^{\pi} \cdot \w \geq \vect{v}^{\pi'} \cdot \w \},
\end{equation}
\noindent where $\w$ is a vector of weights used by the scalarization function $u$.
Hence, the optimal solution to a MOMDP, under linear preferences, is a finite convex subset of the Pareto frontier.

\subsection{Generalized Policy Improvement}

\textit{Generalized Policy Improvement} (GPI) is a generalization of the \textit{policy improvement} step \cite{Puterman2005}, which underlies most of the RL algorithms. The key difference is that GPI defines a new policy that improves over a \textit{set} of policies, instead of a single one.
GPI was originally proposed to be employed in the setting where policies are evaluated using their \textit{successor features} \cite{Barreto+2017,Barreto+2020}. However, Alegre et al.~\cite{Alegre+2022} showed that GPI can also be employed with multi-objective action-value functions (Eq.~\eqref{eq:mo-q-function}).

Let $\Pi = \{\pi_i\}_{i=1}^{n}$ be a set of previously-learned policies with corresponding multi-objective action-value functions $\{\vect{q}^{\pi_i}\}_{i=1}^{n}$.
Given any weight vector $\w {\in} \Wspace$, we can directly evaluate all policies $\pi_i\in\Pi$ by computing the dot product of its multi-objective action-values and the weight vector: $q^{\pi_i}_{\w}(s,a) = \vect{q}^{\pi_i}(s,a) \cdot \w$.
This step is known as \textit{generalized policy evaluation} (GPE) \cite{Barreto+2020}.
We define the \textit{GPI policy} as a policy $\pi: \Sspace\times\Wspace \mapsto \Aspace$ that is constructed from a set of policies $\Pi$, and is conditioned on a weight vector $\w \in \Wspace$:
\begin{equation}
\label{eq:sf-gpi}
    \pigpi(s; \w) \in \argmax_{a\in\Aspace} \max_{\pi \in \Pi} q^{\pi}_{\w}(s,a).
\end{equation}
Let $q^{\text{GPI}}_{\w}(s,a)$ be the action-value function of policy $\pigpi(\cdot;\w)$.
The GPI theorem \cite{Barreto+2017} ensures that $q^{\text{GPI}}_{\w}(s,a) \geq \max_{\pi \in \Pi} q_{\w}^{\pi}(s,a)$ for all $(s,a) \in \mathcal{S} \times \mathcal{A}$.
In other words, GPI allows for the rapid identification of a policy that is guaranteed to perform at least as well as any of the policies $\pi_i\in\Pi$, for any given weight vector $\w\in\Wspace$.
This result is also valid and can be extended to the scenario where  $q^{\pi_i}$ is replaced with estimates/approximations, $\tilde{Q}^{\pi_i}$ \cite{Barreto+2018}.

\subsection{Model-Based RL}

In model-based RL \cite{Moerland+2020}, an agent learns---based on experiences collected while interacting with the environment---an approximate model $\tilde{p}(s',r|s,a)$ of the joint distribution over the next states and rewards.
This model can be employed in multiple ways, such as: \textit{(i)} to perform Dyna-style planning, i.e., generate simulated experiences to more rapidly learn a policy or value function \cite{Seijen&Sutton2013,Janner+2019,Pan+2019}; \textit{(ii)} produce improved model-augmented update targets for use by  temporal-difference algorithms~\cite{Feinberg+2018,Buckman+2018,Abbas+2020}; \textit{(iii)} to perform planning, online, via Model Predictive Control techniques~\cite{Chua+2018}; and \textit{(iv)} to exploit model gradients with respect to its parameters to improve the efficiency of policy parameter updates \cite{Deisenroth&Rasmussen2011,Doro&Jaskowski2020}.
One widely-used model-based framework is based on the Dyna architecture \cite{Sutton1990}, by which an agent learns a model and updates a value function/policy using real and model-generated, simulated experiences.
This significantly improves sample efficiency, as demonstrated by recent successes of RL when using expressive function approximators~\cite{Pan+2018,Janner+2019,Kaiser+2020}.
The key insight is that performing additional updates, using model-generated experiences, reduces the need to frequently interact with the environment, which may be expensive and risky in real-world applications.

A crucial component of Dyna-style algorithms is the \textit{search-control} mechanism, which prioritizes and samples (from the learned model) experiences to be used in Dyna planning.
The classic Dyna algorithm samples state-action pairs \textit{uniformly}. 
Other variants have been studied. For example, the Prioritized Sweeping algorithm~\cite{Moore&Atkeson1993} prioritizes states or state-action pairs with higher absolute \textit{temporal-difference error} (TD-error) $\delta_t$ when performing value updates. 
A similar idea has been extensively investigated in models that are based, e.g., on Prioritized Experience Replay buffers \cite{Lin1992,Schaul+2016,Fujimoto+2020}, which are a form of (limited) non-parametric models \cite{Seijen&Sutton2015}.
According to the Prioritized Experience Replay method, the probability of sampling a given state-action pair, $P(s,a)$, is defined by
\begin{align}
    \label{eq:tderror}
    P(S_t,A_t) \propto |\delta_t|, \quad \delta_t = R_t + \gamma \max_{a'\in\Aspace} q^\pi(S_{t+1},a') - q^\pi(S_t,A_t).
\end{align}
Intuitively, states with higher TD-error are more likely to cause larger changes in post-update value functions, and therefore are also likely to accelerate policy learning.

\section{Sample-Efficient MORL}
\label{sec:sample-eff-MORL}

Since our objective is to improve sample efficiency in MORL settings (and given the observations in Section~\ref{sec:morl_intro}), we propose to design a method to rapidly learn a set of policies $\Pi$ whose values approximate a CCS (Eq.~\eqref{eq:ccs}).
In this Section, we first introduce a novel method based on GPI to identify the most promising preferences/weight vectors $\w\in\Wspace$ to train on, at each moment, while constructing a CCS (Section~\ref{sec:gpi-ls}).
Then, we introduce a principled technique to prioritize experiences (for use when optimizing a given preference) based on the formally-guaranteed improvements achievable via a one-step GPI process (Section~\ref{sec:per-gpi}).
These contributions are combined in Section~\ref{sec:gpi-pd} to derive an effective algorithm that approximates the CCS in discrete- and continuous-state settings.

\subsection{Prioritizing Weight Vectors via GPI}
\label{sec:gpi-ls}

In this Section, we introduce a principled algorithm that iteratively constructs a set of policies, $\Pi$, whose value vectors, $\Vset$, approximate the CCS.
At each iteration, it selects a weight vector $\vect{w}\in\Wspace$ based on the formally-guaranteed improvements achievable via GPI, and learns a new policy $\pi_{\w}$ specialized in optimizing $\w$.

{\small
\begin{algorithm}[ht]
\caption{\textbf{GPI} \textbf{L}inear \textbf{S}upport (\textbf{GPI-LS})}
\label{alg:algo1}

\DontPrintSemicolon
\SetKwInOut{Input}{Input}
\Input{MOMDP $M$}

$\pi_{\w}, \vect{v}^{\pi_{\w}} \leftarrow \mathrm{NewPolicy}(\w=[1,0,...,0]^\top)$\;
$\Pi \leftarrow \{\pi_{\w}\}, \Vset \leftarrow \{\vect{v}^{\pi_{\w}}\}$,
$\Wsupport \leftarrow \{\}$\;
\While{True}{
    $\Wspace_{\mathrm{corner}} \leftarrow \mathrm{CornerWeights(\Vset)} \setminus \Wsupport$\;
    \If{$\Wspace_{\mathrm{corner}}$ is empty}{
       \Return{$\Pi, \Vset$}; \Comment*[r]{Found CCS (or $\epsilon$-CCS)}
   }

    $\w \leftarrow \argmax_{\w\in\Wspace_{\mathrm{corner}}}
   (v^{\gpi}_{\w} - \max_{\pi\in\Pi} v^{\pi}_{\w})$\;
   
   $\pi_{\w},\vect{v}^{\pi_{\w}}, \mathrm{done} \leftarrow \mathrm{NewPolicy}(\w, \Pi)$\;
   \If{$\mathrm{done}$}{
       Add $\w$ to $\Wsupport$; \Comment*[r]{Adds $\w$ to support of partial CCS}
   }
      Add $\pi_{\w}$ to $\Pi$ and $\vect{v}^{\pi_{\w}}$ to $\Vset$\;
      $\Pi, \Vset \leftarrow \mathrm{RemoveDominated}(\Pi,\Vset)$\;
}

\end{algorithm}
}

We start by defining the \textit{corner weights}~\cite{Roijers2016} of a given set of value vectors, $\Vset$. Corner weights $\Wspace_{\mathrm{corner}} \subset \Wspace$ are a finite subset of the weight vector space with relevant properties we will explore.
\begin{definition}
Let $\Vset = \{\vect{v}^{\pi_i}\}_{i=1}^{n}$ be a set of multi-objective value functions of $n$ policies. \textit{Corner weights} are the weights contained in the vertices of a polyhedron, $P$, defined as:
\begin{equation}
\label{eq:polyhedron-corner-weights}
    P = \{ \vect{x} \in \mathbb{R}^{d+1} \ | \ \vect{V}^{+}\vect{x} \leq \vect{0}, \textstyle\sum_i w_i = 1, w_i \geq 0, \forall i \},
\end{equation}
where $\vect{V}^{+}$ is a matrix whose rows store the elements of $\Vset$ and is augmented by a column vector of $-1$'s. 
Each vector $\vect{x} {=} (w_1, ..., w_d, v_{\w})$ in $P$ is
composed of a weight vector and its scalarized value.
\end{definition}
Intuitively, corner weights are the weight vectors for which the policy selected in the maximization $\max_{\pi\in\Pi} v^{\pi}_{\w}$ changes. These are weight vectors for which two or more policies in $\Pi$ share the same value with respect to the above-mentioned maximization.
Extrema weights (weights where only one element is 1 and all others are 0) are special cases of corner weights.
Notice that the number of vertices of the polyhedron $P$ is finite when $\Vset$ is finite.
The importance of corner weights comes from the following theorem:
\begin{theorem}
\label{th:corner-weights}
(Theorem 7 of Roijers~\cite{Roijers2016}). Let $\Pi=\{\pi_i\}_{i=1}^{n}$ be a set of $n$ policies with corresponding value vectors $\Vset = \{\vect{v}^{\pi_i}\}_{i=1}^{n}$. Let $\Delta(\w,\Pi) = v^{*}_{\w} - \max_{\pi\in\Pi}v^{\pi}_{\w}$ be the \textit{utility loss} of weight vector $\w \in \Wspace$ given the policy set $\Pi$; that is, the difference between the value of the optimal policy for $\w$ and the value that can be obtained if using one of the policies in $\Pi$ for solving $\w$. Then, a weight vector $\w \in \arg\max_{\w \in \Wspace} \,\,\Delta(\w,\Pi)$ is one of the \textit{corner weights} of $\Vset$.
\end{theorem}
Due to this theorem, when selecting a weight vector to train on when constructing a CCS, we only need to consider the finite set of corner weights, $\Wspace_{\mathrm{corner}}$, instead of the infinite weight simplex $\Wspace$. However, this does tell us \textit{how} to select a weight vector (among all the corner weights) to more rapidly learn the CCS.
Let $v^{\gpi}_{\w}$ be the scalarized value of the GPI policy (Eq.~\ref{eq:sf-gpi}) for the weight vector $\w$.
We propose to prioritize weight vectors based on the magnitude of the improvement that can be formally achieved via GPI.
In particular, given the corner weights $\Wspace_{\mathrm{corner}}$ of the current value set, $\Vset$, we select the weight vector $\w{\in}\Wspace_{\mathrm{corner}}$ given by:
\begin{eqnarray}
\label{eq:best-w}
   \argmax_{\w\in\Wspace_{\mathrm{corner}}} 
   (v^{\gpi}_{\w} - \max_{\pi\in\Pi} v^{\pi}_{\w}).
\end{eqnarray}
Intuitively, Eq.~\eqref{eq:best-w} identifies the corner weight for which we are guaranteed to achieve the maximum possible improvement via GPI.

In Alg.~\ref{alg:algo1} we show an iterative algorithm, \textbf{GPI} \textbf{L}inear \textbf{S}upport (\textbf{GPI-LS}), that constructs a CCS in a finite number of iterations based on the above-mentioned ideas.\footnote{A key challenge in proving convergence and bounded utility loss of our algorithm is that, unlike existing techniques, it does \textit{not} require optimal (or $\epsilon$-optimal) policies to be returned at every iteration. Instead, it can incrementally (and monotonically) improve its CCS's quality even when given partially learned, possibly suboptimal policies.}
Let $\mathrm{NewPolicy}(\w,\Pi)$ be any RL algorithm that searches for a policy that optimizes a weight vector $w{\in}\Wspace$, starting from an initial candidate policy $\pi_{\w}$ set to $\argmax_{\pi\in\Pi} v^{\pi}_{\w}$. We assume this algorithm  returns the value vector of this policy, $\vect{v}^{\pi_{\w}}$, and a flag ($\mathrm{done}$) indicating if it has reached its stopping criterion.
When this happens, it adds the corresponding new policy to $\Pi$ and removes all dominated policies; i.e., policies whose value vectors are no longer the best for any weight vector.
To analyze this algorithm, we first consider the best-case scenario where the $\mathrm{NewPolicy(\w,\Pi)}$ is capable of identifying optimal policies for any given weight vectors. Later, we relax this assumption and show that our algorithm's theoretical guarantees still hold even if that is not the case.
We show that it holds strong theoretical guarantees bounded only by the sub-optimality of the underlying RL algorithm chosen by the user.

\begin{theorem}
\label{th:guaranteeCCS}
Let $\mathrm{NewPolicy}(\w,\Pi)$ in Alg.~\ref{alg:algo1} be any algorithm that returns an optimal policy, $\pi^{*}_{\w}$, for a given weight vector $\w$.
Then, Alg.~\ref{alg:algo1} is guaranteed to find a $\ccs$ in a finite number of iterations.
\end{theorem}
\begin{proof}
The proof follows from the fact that at each iteration, the set of corner weights, $\Wspace_{\mathrm{corner}}$, is finite since the number of the vertices of the polyhedron $P$ is finite (Eq.~\eqref{eq:polyhedron-corner-weights}).
Once the termination condition of $\mathrm{NewPolicy}(\w,\Pi)$ is met ($\mathrm{done}=\mathrm{True}$), it returns the optimal policy for a given $\w\in\Wspace_{\mathrm{corner}}$, $\w$ is added to $\Wsupport$, and never selected again. This is ensured because the set of candidate corner weights, $\Wspace_{\mathrm{corner}}$, computed by GPI-LS in any given iteration, does not (by construction) include previously-evaluated corner weights.
Since $\Sspace$ and $\Aspace$ are finite, the number of potential corner weights is bounded by the (finite) number of vertices of the polyhedron $P$ defined with respect to the set $\Vset$ containing all $|\Aspace|^{|\Sspace|}$ possible policies.
Because \textit{(i)} GPI-LS selects a corner weight at each iteration; \textit{(ii)} $\mathrm{NewPolicy}(\w,\Pi)$ returns an optimal policy for $\w$, which is added to $\Wsupport$ and never selected again; and \textit{(iii)} there exists a finite number of corner weights, it follows that in the worst-case, GPI-LS will test all (finite) corner weights. Thus, after a finite number of iterations, all corner weights of the CCS will be in $\Wsupport$.
At this point, $\Wspace_{\mathrm{corner}}$ will be empty and the algorithm will return $\Pi$ and $\Vset$ (lines 5--6).
We can ensure that the returned set $\Vset$ is a CCS due to Theorem~\ref{th:corner-weights}. When there are no more corner weights to be analyzed ($\Wspace_{\mathrm{corner}}$ is empty), $\Wsupport$ contains all corner weights and their corresponding optimal policies.
By Theorem~\ref{th:corner-weights}, this means that it is not possible to make improvements to any weights vector $\w\in\Wspace$ and thus $\Vset$ (and its corresponding policies $\Pi$) is a CCS.
\end{proof}

We now relax the assumption that $\mathrm{NewPolicy}$ finds optimal solutions and show that if $\mathrm{NewPolicy}(\w,\Pi)$ converges to a local-minimum (an $\epsilon$-optimal policy), GPI-LS returns an $\epsilon$-CCS.\footnote{Notice that when the CCS is infinite (e.g., there is an infinite number of optimal policies), GPI-LS still converges to a CCS asymptotically/in the limit.}
\begin{definition}
\label{def:epsilon-ccs}
A set of value vectors $\Vset = \{\vect{v}^{\pi_i}\}_{i=1}^{n}$ (associated with policy set $\Pi=\{\pi_i\}_{i=1}^{n}$) is an $\epsilon$-CCS if, for all weight vectors $\w \in \Wspace$, the corresponding utility loss $\Delta(\w,\Pi)$ is at most $\epsilon$:
\begin{align*}
    \max_{\w\in\Wspace}\Delta(\w,\Pi) =
    \max_{\w\in\Wspace} (v^*_{\vect{w}} - \max_{\pi\in\Pi} v^{\pi}_{\w}) \leq \epsilon .
\end{align*}
\end{definition}

Intuitively, an $\epsilon$-CCS is a convex coverage set whose maximum utility loss is at most $\epsilon$; i.e., based on it, for \textit{any} possible weight vectors $w{\in}\Wspace$, it is possible to identify a policy whose value differs from the value of the optimal policy for $w$ by at most $\epsilon$.
\begin{theorem}
\label{th:guaranteeEpsCCS}
Let $\mathrm{NewPolicy}(\w,\Pi)$ in Alg.~\ref{alg:algo1} be an algorithm that produces an $\epsilon$-optimal policy, $\pi_w$, when its termination condition is met (when it returns $\mathrm{done}=\mathrm{True}$); that is, $v^{*}_{\w} - v^{\pi_{\w}}_{\w} \leq \epsilon$. Then, Alg.~\ref{alg:algo1} is guaranteed to return an $\epsilon$-$\ccs$.
\end{theorem}
\begin{proof}
Let $\Pi_i$ be the set of policies computed by GPI-LS at iteration $i$ and $\mathcal V_i$ be the corresponding set of value functions. Let $\Delta_i = \,\,\max_{\w \in \Wspace} \Delta(\w, \Pi_i)$ be the maximum utility loss incurred by the partial CCS computed by GPI-LS at iteration $i$.
If $\Delta_i \leq \epsilon$, then by definition the set of value vectors $\Vset_i$, computed by GPI-LS at iteration $i$, is an $\epsilon$-CCS.
If not, then by Theorem~\ref{th:corner-weights} there exists a corner weight $\w'$ such that $\Delta(\w',\Pi_i) = \Delta_i$.
Because GPI-LS will, in the worst-case, test all corner weights and $\mathrm{NewPolicy}$ returns $\epsilon$-optimal policies (where $\epsilon{<}\Delta_i)$, then at some iteration $j{>}i$, GPI-LS will select $\w'$, compute its optimal policy ($\pi_{\w'}$) and add it to $\Pi_j$. Because by definition $\Delta(\w',\Pi_j) \leq \epsilon$, it follows that $\Delta_j {<} \Delta_i$; that is, the maximum utility loss incurred by the partial CCS can be strictly decreased. Thus, within a finite number of iterations, $k$, we can guarantee that $\Delta_{i+k} \leq \epsilon$. At this point, by definition, an $\epsilon$-CCS has been identified.
\end{proof}

Finally, we introduce a formal bound that characterizes the maximum utility loss (with respect to the optimal solution) incurred by  \textit{any partial} solutions computed by our method, prior to convergence, in case the agent follows a GPI policy based on the partially-computed CCS produced by our algorithm at an arbitrary iteration.\footnote{This result is a straightforward adaptation of Theorem 2 of Barreto et al.~\cite{Barreto+2017} to the setting where we wish to bound the maximum utility loss of a partially-computed CCS.} 

\begin{theorem}
Let $\Vset = \{\vect{v}^{\pi_i}\}_{i=1}^{n}$ be a set of value vectors with corresponding policies $\Pi=\{\pi_i\}_{i=1}^{n}$ optimizing weight vectors $\Wset=\{\w_i\}_{i=1}^n$.
Let $
    |q^{*}_{\w_i}(s,a) - q^{\pi_{i}}_{\w_i}(s,a)| \leq \delta \text{ for all } (s,a) \in \Sspace{\times}\Aspace, \, and\, i \, \in \{1,...,n\}. 
$
Let $\vect{r}_{\mathrm{max}} = \max_{s,a}||\vect{r}(s,a)||$. Then, for any $\w \in \Wspace$,
\begin{equation}
\label{eq:bound-gpi}
    q^{*}_{\w}(s,a) - q^{\gpi}_{\w}(s,a) \leq \frac{2}{1-\gamma} (\vect{r}_{\mathrm{max}} \min_{i}||\w-\w_i|| + \delta).
\end{equation}
\end{theorem}
This theorem guarantees that as the coverage of the weight vector set, $\Wspace$, increases with respect to the simplex, the maximum utility loss of Alg.~\ref{alg:algo1} is strictly bounded at any given iteration during learning.

\subsection{Prioritizing Experiences via GPI}
\label{sec:per-gpi}

A complementary approach for increasing sample efficiency is to use a model-based RL algorithm to learn policies for different preferences. In this paper, we introduce an experience prioritization technique based on a Dyna-style MORL algorithm (introduced and discussed later). An important open question related to efficiently deploying Dyna algorithms is to determine \textit{which} artificial, model-simulated experiences should be generated to accelerate learning. We introduce a new, principled GPI-based prioritization technique (based on Theorem~\ref{th:gpi-equality}) for identifying which experiences are most relevant to rapidly learn the optimal policy, $\pi_{\w}$, that optimizes a preference $\w \in \Wspace$. Importantly, unlike similar existing methods (e.g., Prioritized Experience Replay~\cite{Schaul+2016}), our method is \textit{specially} designed to accelerate learning in MORL problems.

We start by introducing a theorem based on which we will be able to define a principled experience prioritization scheme for accelerating learning in MORL problems.
\begin{theorem}
\label{th:gpi-equality}
Let $\Pi$ be an arbitrary set of policies, and let $\pi_{\w}$ in $\Pi$ be a deterministic policy tasked with optimizing some $\w\in\Wspace$. Then, $q^{\gpi}_{\w}(s,a) = q^{\pi}_{\w}(s,a)$ \textbf{\textit{for all}}
state-action pairs in $\Sspace {\times} \Aspace$ if and only if  $q^*_{\w}(s,a) = q^{\pi}_{\w}(s,a)$. In other words, $\pi_{\w}$ is guaranteed to be an optimal policy for $\w$ iff the GPI policy computed over $\Pi$, for optimizing $\w$, cannot improve the $q$-function of $\pi_{\w}$ for any   state-action pairs.
\footnote{The proof of this theorem can be found in the Appendix.}
\end{theorem}
As a result of this theorem, it follows that to rapidly learn an optimal policy, we may wish to bring the value of $q^{\pi}_{\w}(s,a)$ closer $q^{\gpi}_{\w}(s,a)$: when they are equal in all state-action pairs, we have found $q^*_{\w}(s,a)$.
State-action pairs whose value gap, $q^{\gpi}_{\w}(s,a) - q^{\pi}_{\w}(s,a)$, is maximal (or large) are pairs that, if updated, more rapidly approximate $q^{\pi}_{\w}(s,a)$ and $q^{\gpi}_{\w}(s,a)$ (in terms of Max-Norm distance). Intuitively, these are promising candidate experiences to be sampled from the model and used in updates to improve policy $\pi_{\w}$. 
By decreasing the above-mentioned gap, we are guaranteed to move $q^{\pi}_{\w}(s,a)$ closer to $q^*_{\w}(s,a)$. When this gap is zero for all state-actions pairs, we are guaranteed to have identified an optimal policy for $\w$.
Based on these observations, which follow from Theorem \ref{th:gpi-equality}, we propose (during the process of learning $\pi_{\w}$) to prioritize experiences  proportionally to the magnitude of the value improvement resulting from using a GPI policy. That is, for a given weight vector $\w {\in} \Wspace$ and state-action pair, we assign corresponding priorities $P_{\w}(s,a)$ proportionally to the gap $q^{\gpi}_{\w}(s,a) - q^{\pi}_{\w}(s,a)$:
\begin{eqnarray}
   P_{\w}(s,a) \propto q^{\gpi}_{\w}(s,a) - q^{\pi}_{\w}(s,a) .
\end{eqnarray}
Notice that in practice we do not have direct access to the action-values of the GPI policy, $q^{\gpi}_{\w}(s,a)$, unless we evaluate it beforehand with a policy evaluation algorithm (which may be costly).
We can, however, efficiently compute the value of employing GPI for a single time step.
Let $q^{1-\gpi}_{\w}(S_t,A_t)$ be defined as the value of executing action $A_t$ in state $S_t$, following the GPI policy for one step, and thereafter following the same policy $\pi \in \Pi$ that was used in this first step: $q^{1-\gpi}_{\w}(S_t,A_t) = \Ex[\vect{R}_t \cdot \vect{w} + \gamma \max_{a'\in\Aspace}\max_{\pi\in\Pi} q^\pi_{\w}(S_{t+1},a') ]$.
It is possible to show that $q^{\gpi}_{\w}(s,a) \geq q^{1-\gpi}_{\w}(s,a) \geq q^{\pi}_{\w}(s,a)$.
Importantly, Theorem~\ref{th:gpi-equality} \textit{still holds} if we replace $q^{\gpi}_{\w}(s,a)$ with $q^{1-\gpi}_{\w}(s,a)$. Given a transition $(S_t,A_t,\vect{R}_t,S_{t+1})$ and a weight vector $\w\in\Wspace$, we thus compute experience priorities as follows:
\begin{equation}
\label{eq:pr-1gpi}
    P_{\w}(S_t,A_t) \propto
    |\vect{R}_t \cdot \vect{w} + \gamma \max_{a'\in\Aspace}\max_{\pi\in\Pi} q^\pi_{\w}(S_{t+1},a')  -  q^{\pi}_{\w}(S_t,A_t)|
    .
\end{equation}
Notice that when $|\Pi|=1$, that is, when we only have one policy, Eq.~\eqref{eq:pr-1gpi} reduces to the commonly-used Prioritized Experience Replay prioritization scheme  (Eq.~\eqref{eq:tderror}).
Hence, we can see Eq.~\eqref{eq:pr-1gpi} as a form of \textit{generalized TD-error-based prioritization}.

\section{GPI-Prioritized Dyna}
\label{sec:gpi-pd}

We now introduce \textbf{GPI} \textbf{P}rioritized \textbf{D}yna (\textbf{GPI-PD}), a novel model-based MORL algorithm that improves sample efficiency by simultaneously using GPI-LS (Section~\ref{sec:gpi-ls}) to prioritize and select weight vectors to more rapidly construct a CCS, and our GPI-based experience prioritization scheme (Section~\ref{sec:per-gpi}) to more rapidly learn policies for a given preference.
GPI-PD learns an approximate multi-objective dynamics model, $\model$, predicting the next state and reward vector given a state-action pair.
This model is used to perform Dyna updates to multi-objective action-value functions via model-generated, simulated experiences.
Its pseudocode is shown in Alg.~\ref{alg:gpi-pd}. 

{\small
\begin{algorithm}[ht]
\caption{\textbf{GPI}-\textbf{P}rioritized \textbf{D}yna (\textbf{GPI-PD})}
\label{alg:gpi-pd}

\DontPrintSemicolon
\SetKwInOut{Input}{Input}
\SetKwProg{everyn}{}{ do}{}

Initialize action-value function $\vect{Q}_{\theta}(s,a,\w)$, dynamics model $\model_{\varphi}$, buffers $\buffer$ and $\buffer_{\mathrm{model}}$, weight support set $\Wset$\;

$\Wset \leftarrow$ extrema weights of $\Wspace$; $\w_0 \sim \Wset$\;
\For{$t = 0 ... \infty$}{

\everyn(\Comment*[f]{GPI Linear Support (Alg.~\ref{alg:algo1})}){Every $N$ time steps}{ 
    $\Vset \leftarrow$ evaluate $\vect{v}^{\pi_{\w}}$ for all $\w\in\Wset$\;
    $\Wset, \Vset \leftarrow \mathrm{RemoveDominated(\Wset,\Vset)}$\;
    $\Wspace_{\mathrm{corner}} \leftarrow \mathrm{CornerWeights(\Vset)}$\;
    Add to $\Wset$ the top-$k$ weight vectors in $\Wspace_{\mathrm{corner}}$ w.r.t. $\argmax_{\w\in\Wspace_{\mathrm{corner}}}
   (v^{\gpi}_{\w} - \max_{\pi\in\Pi} v^{\pi}_{\w})$\;
}

\If{$S_t$ is terminal}{
    $\w_t \sim \Wset$\;
    $S_t \sim \mu$\;
}

$A_t \leftarrow \pigpi(S_t; \w_t)$  \ (Eq.~\eqref{eq:fa-gpi}) \Comment*[r]{Follow GPI policy }

Execute $A_t$, observe $S_{t+1}$, and $\vect{R}_t$\;
Add $(S_t,A_t,\vect{R}_t,S_{t+1})$ to $\buffer$ with priority $P_{\w_t}(S_t,A_t)$\;

Update model $\model_{\varphi}$ with experience tuples from $\buffer$\;

\For(\Comment*[f]{GPI-Prioritized Dyna}){$H$ Dyna steps}{
    Sample $S \sim \buffer$ according to $P_{\w_t}$ \quad (Eq.~\eqref{eq:pr-1gpi})\;
    $A \leftarrow \pigpi(S;\w_t)$;
    $(\hat{S}', \hat{\vect{R}}) \sim \model_{\varphi}(\cdot|S,A)$\;
    Add $(S,A,\hat{\vect{R}},\hat{S}')$ to $\buffer_{\mathrm{model}}$\;
}
\vspace{0.1cm}
\Comment*[l]{Update multi-objective Q-function}
\For{$G$ gradient updates}{
   Build mini-batch $\{(S_i,A_i,\vect{R}_i,S'_i)\}_{i=1}^{b}$ with $\beta b$ tuples from $\buffer_{\mathrm{model}}$ and $(1-\beta)b$ tuples from $\buffer$\;
   
   Update $\vect{Q}_{\theta}$ by minimizing  $\mathcal{L}(\theta;\w_t) + \mathcal{L}(\theta;\w')$ w.r.t. $\theta$ via mini-batch gradient descent, where $\w'\sim\Wset$\;
   
   \mbox{Update priorities $P_{\w_t}$ of all pairs $(S_i,A_i)$ in the mini-batch}
}

}
\end{algorithm}
}

Notice that GPI-PD is a MORL algorithm designed to deal with high-dimensional and continuous state spaces; it models the action-value function via a single neural network ($\vect{Q}_{\theta}$) conditioned on state and weight vector, as in \cite{Abels+2019,Yang+2019}.
The multi-objective action-value function is defined as $\vect{Q}_{\theta}(s,a,\w) \approx \vect{q}^{\pi_{\w}}(s,a)$, where $\theta$ are the learned neural network parameters.
Similarly to \cite{Borsa+2019}, we now rewrite the GPI policy definition (Eq.~\ref{eq:sf-gpi}) so that the GPI policy can be defined with respect to an approximator:
\begin{eqnarray}
\label{eq:fa-gpi}
   \pigpi(s;\w) \in \argmax_{a\in\Aspace}\max_{\w'\in\Wset} \vect{Q}_{\theta}(s,a,\w') \cdot \w,
\end{eqnarray}
where $\Wset$ is the weight support set containing the corner weights selected by GPI-LS (lines 4--8 of Alg.~\ref{alg:gpi-pd}). 
Notice that we add to $\Wsupport$ the top-$k$ weight vectors, according to Eq.~\eqref{eq:best-w} (line 8). This implies that each top-performing weight will be optimized during a corresponding episode. We found that this accelerates learning in the function approximation setting. Finally, the parameters $\theta$ of $\vect{Q}_{\theta}$ are updated by minimizing the mean-squared error of the multi-objective TD error:
\begin{eqnarray}
\label{eq:loss-q}
  \mathcal{L}(\theta;\w) = \Ex_{(S,A,\vect{R},S')\sim\buffer}[(\vect{R} + \gamma \vect{Q}_{\theta^{-}}(S',a',\w) - \vect{Q}_{\theta}(S,A,\w))^2],
\end{eqnarray}
where $a' = \argmax_{a'\in\Aspace } \vect{Q}_{\theta^{-}}(S', a', \w) \cdot \w$ and $\theta^-$ are the parameters of a target neural network, which are updated to match $\theta$ periodically.
In line 22 of Alg.~\ref{alg:gpi-pd}, we optimize $\mathcal{L}(\theta;\w)$ both for the current episode's weight vector, $\w_t$, and for weight vectors $\w' {\sim} \Wsupport$ sampled from the weight support set.\footnote{We abuse notation and use the operator $\sim$ both to denote sampling from a given distribution, but also, when writing $x \sim Y$ (where $Y$ is a set and $x \in Y$) to refer to the process of sampling an element $x$ uniformly at random from the set $Y$.} Training in this manner is known to avoid catastrophic forgetting~\cite{Abels+2019}.

\subsection{Dyna with a Learned MOMDP Model}
\label{sec:model}

GPI-PD learns a model $\model_{\varphi}$ of the environment's dynamics and its multi-objective reward function. In particular, it represents the joint distribution, $\model_{\varphi}(S_{t+1},\vect{R}_t|S_t,A_t)$, of the next state and the \textit{complete multi-objective reward vector}, $\vect{R}_t$, when conditioned on state and action. Notice that this model may be used to accelerate learning of policies for \textit{any} $\w \in \Wspace$.

Before proceeding, we emphasize that GPI-PD is designed to tackle high-dimensional MORL problems with continuous state spaces. To achieve this, we extend state-of-the-art model-based algorithms used in the single-objective setting~\cite{Chua+2018,Janner+2019,Lai+2021}. We achieved this by employing an ensemble of probabilistic neural networks for predicting the reward vector, $\vect{R}_t$, instead of a scalar reward.
The learned model $\model_{\varphi}$ is composed of an ensemble of $n$ neural networks, $\{\model_{\varphi_i}\}_{i=1}^{n}$, each of which outputs the mean and diagonal covariance matrix of a multivariate Gaussian distribution:
\begin{equation}
    \model_{\varphi_i}(S_{t+1},\vect{R}_t \ | \ S_t,A_t) = \mathcal{N}(\mu_{\varphi_i}(S_t,A_t), \Sigma_{\varphi_i}(S_t,A_t)) .
\end{equation}
Each model in the ensemble is trained in parallel to minimize the following negative log-likelihood loss function, using different bootstraps of experiences in the buffer $\buffer$: 
\begin{equation}
    \mathcal{L}(\varphi) = \Ex_{(S_t,A_t,\vect{R}_t,S_{t+1})\sim\buffer}[ -\log \model_{\varphi}(S_{t+1},\vect{R}_t|S_t,A_t)].
\end{equation}
In lines 16--19 of Alg.~\ref{alg:gpi-pd}, GPI-PD samples states according to $P_{\w}(\cdot, A_t)$, where $A_t$ is the action given by the GPI policy. It then generates corresponding model-based simulated experiences for this state-action pair and stores them in the buffer $\buffer_{\mathrm{model}}$. We combine these experiences with ``real'' ones, collected by the agent while interacting with the environment (and stored in the buffer $\buffer$), according to a mixing ratio $\beta$. This results in a mini-batch of experiences that are used to optimize $\vect{Q}_{\theta}$. Combining experiences in this manner is a commonly-used approach when training model-based RL algorithms~\cite{Janner+2019,Pan+2022}.

\section{Experiments}
\begin{figure}[ht!]
    \centering
    \includegraphics[width=0.25\linewidth]{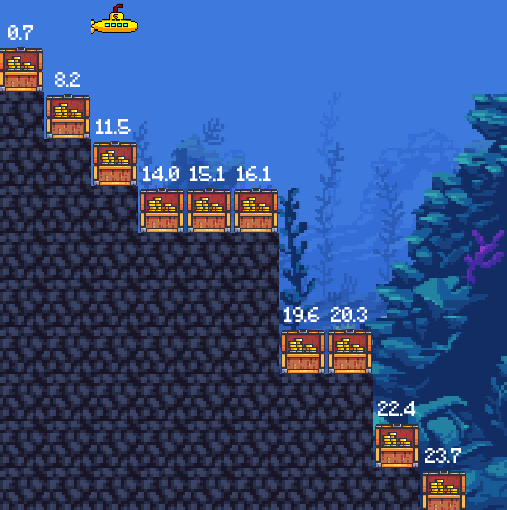}
    \includegraphics[width=0.25\linewidth]{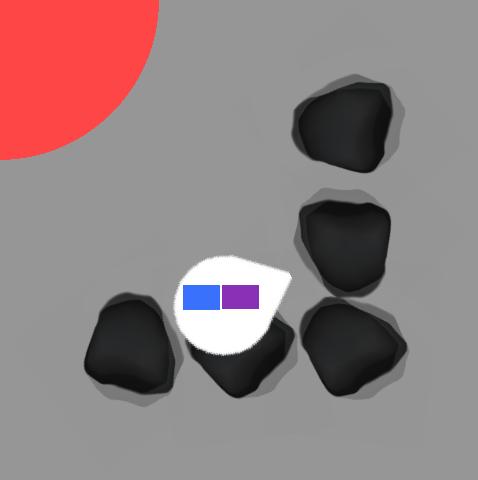}
    \includegraphics[width=0.25\linewidth]{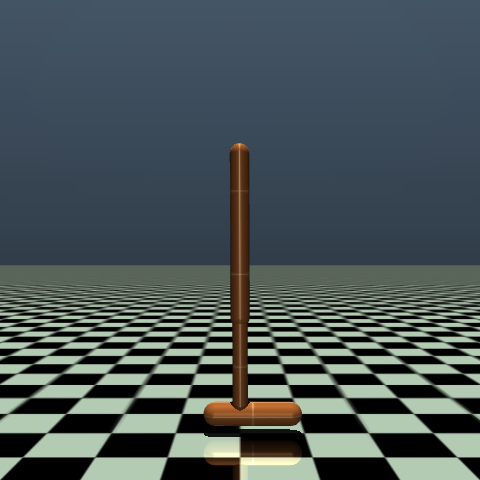}
    \caption{(a) Deep Sea Treasure; (b) Minecart; (c) MO-Hopper.}
    \label{fig:envs}
\end{figure}
\begin{figure*}[ht!]
    \centering
    \includegraphics[width=0.247\linewidth]{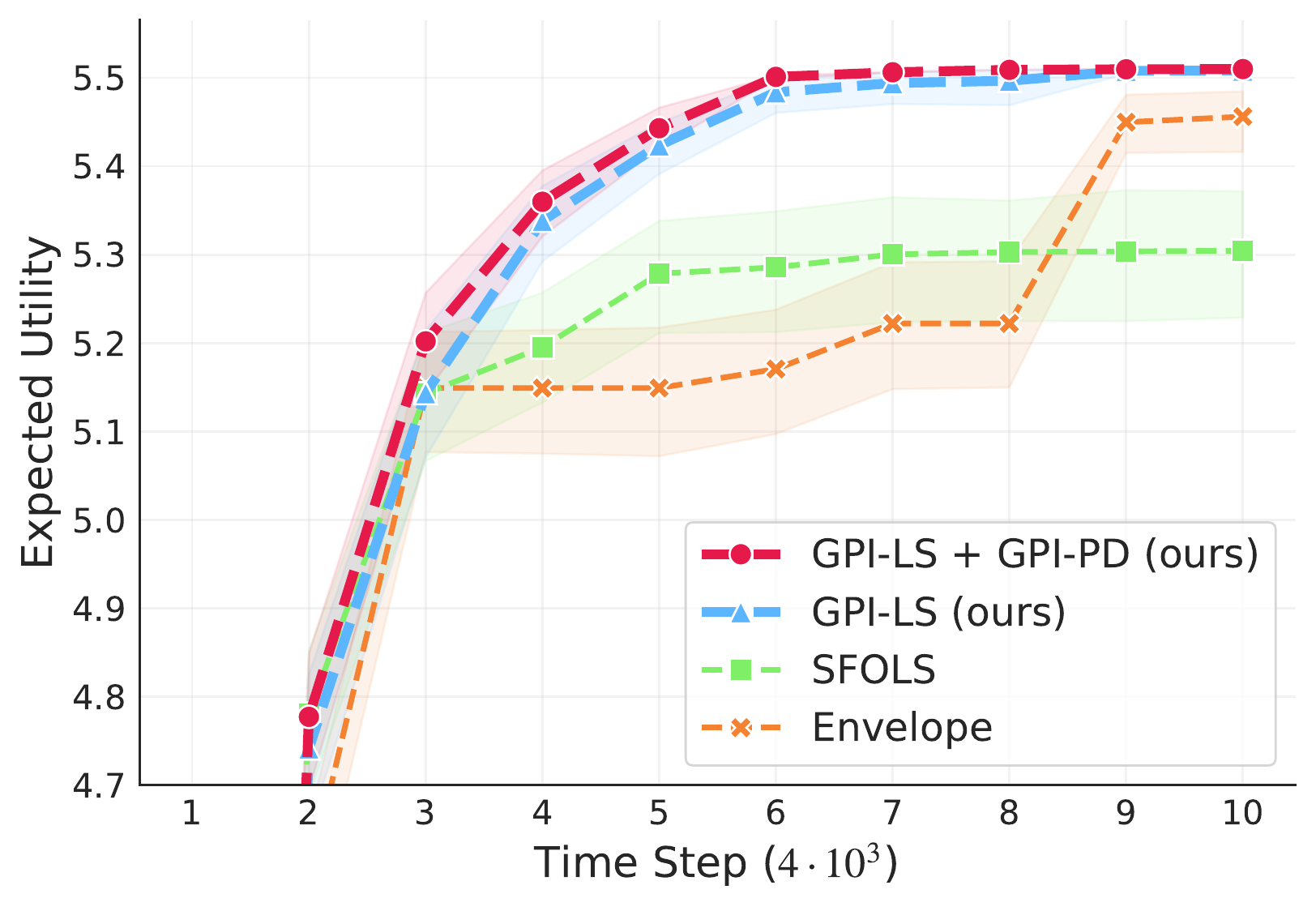}
     \includegraphics[width=0.247\linewidth]{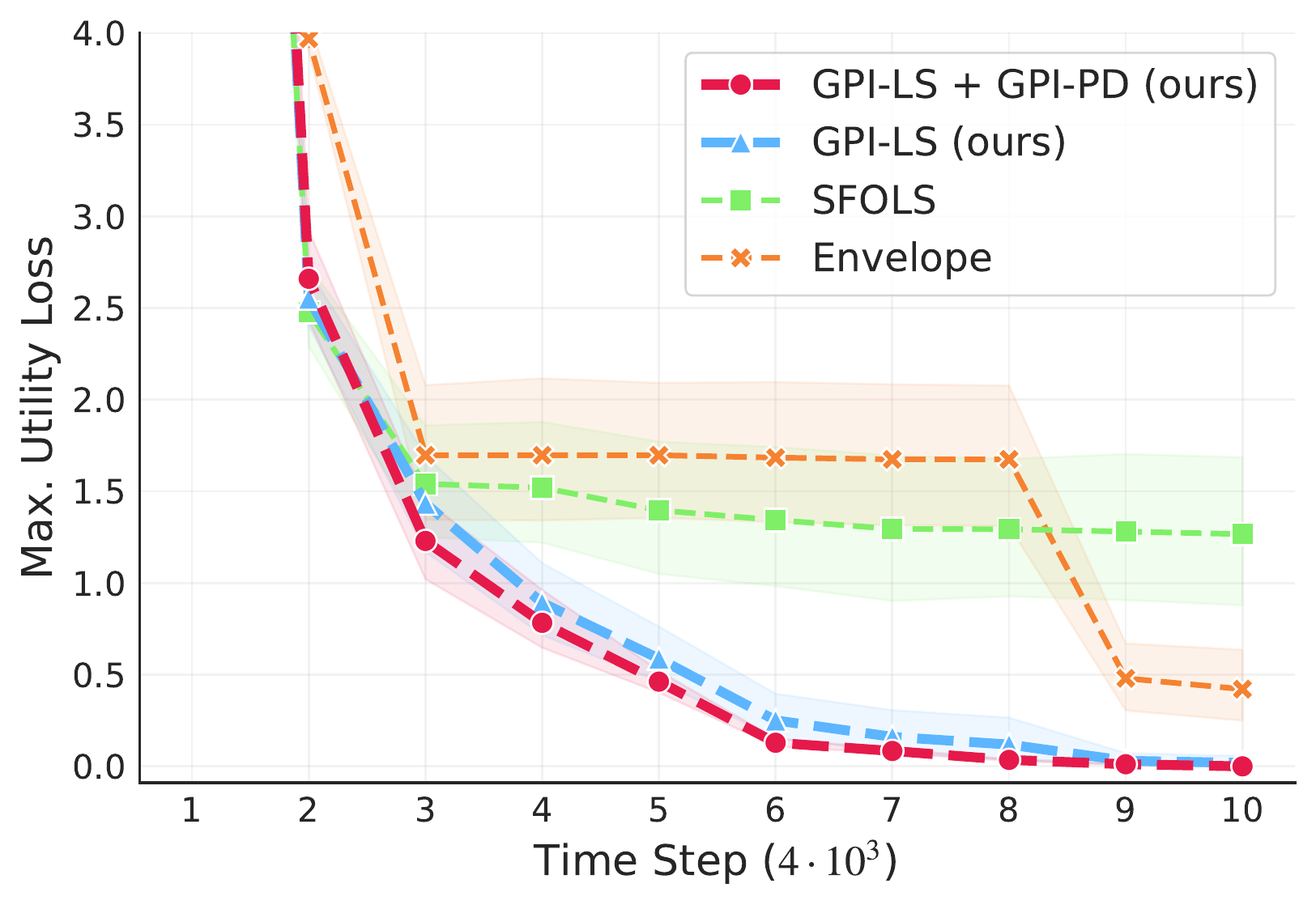}
    \includegraphics[width=0.247\linewidth]{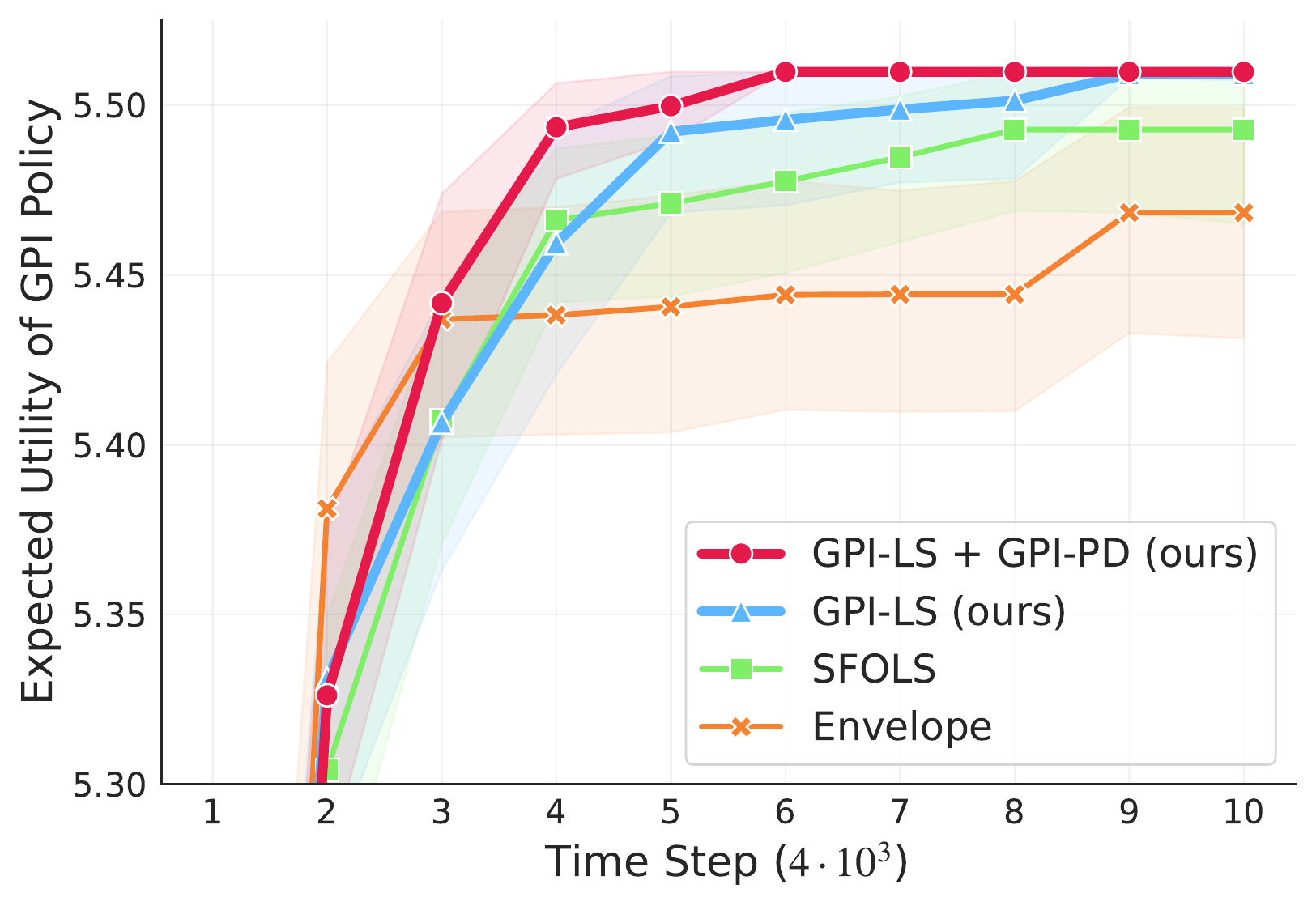}
    \includegraphics[width=0.247\linewidth]{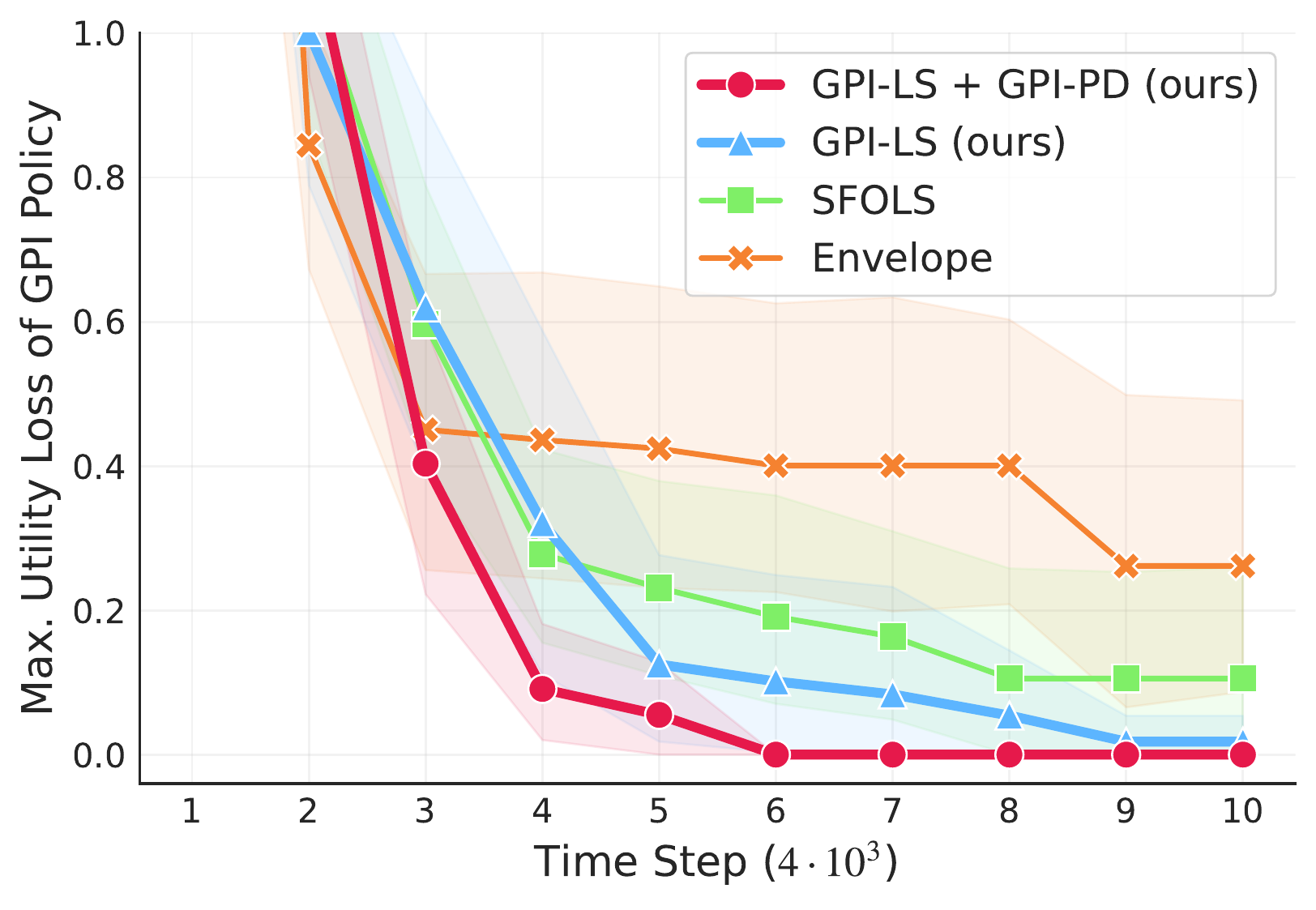}
    \caption{[Deep Sea Treasure Domain] Expected Utility and Maximum Utility Loss of each algorithm. Leftmost plots: performance of the best known policy. Rightmost plots: performance when following the GPI policy with respect to the known policies.}
    \Description{}
    \label{fig:dst}
\end{figure*}

We evaluate the performance of our algorithms in challenging multi-objective tasks---both with discrete and continuous state and action spaces---and compare it with state-of-the-art MORL algorithms. We consider three environments~\cite{Alegre+2022bnaic} with qualitatively different characteristics (see Fig.~\ref{fig:envs}). We follow the performance evaluation methodology proposed in~\cite{Zintgraf+2015,Hayes+2022}, which focuses on \textit{utility-based} metrics. These are relevant since they more directly reflect an algorithm's performance with respect to a user's preferences. Concretely, when evaluating a given algorithm, we compute the \textit{expected utility} of its solutions: $\mathrm{EU}(\Pi) = \Ex_{\w\sim\Wspace}[\max_{\pi\in\Pi} v^{\pi}_{\w}]$.\footnote{We approximate the expectation $\Ex_{\w\in\Wspace}[\cdot]$ by averaging over 100 equidistant weight vectors covering the simplex $\Wspace$.}
We also quantify the \textit{maximum utility loss} of an algorithm: $\mathrm{MUL}(\Pi) = \max_{\w\in\Wspace} (v^{*}_{\w} - \max_{\pi\in\Pi} v^{\pi}_{\w})$.\footnote{When following the GPI policy, EU and MUL can be written as $\mathrm{EU}(\Pi) {=} \Ex_{\w\sim\Wspace}[ v^{\gpi}_{\w}]$ and $\mathrm{MUL}(\Pi) {=} \max_{\w\in\Wspace} (v^{*}_{\w} - v^{\gpi}_{\w})$, respectively.}
We report the mean and its $95\%$ confidence interval over multiple random seeds.\footnote{All algorithm's hyperparameters are presented in the Appendix.}

\vspace{0.1cm}
\noindent\textbf{Deep Sea Treasure.}
This is a classic MORL domain in which a submarine has to trade-off between collecting treasures and minimizing a time penalty  (Fig.~\ref{fig:envs}a).
When evaluating our method in this setting, we employed a simplified, tabular version of GPI-PD (see Appendix). Recall that our GPI-LS algorithm (Alg.~\ref{alg:algo1}), unlike GPI-PD (Alg.~\ref{alg:gpi-pd}), is model-free, does not perform Dyna planning steps, and uses a standard, unprioritized experience replay buffer with uniform sampling.
We compare GPI-PD with two state-of-the-art competitors: Envelope MOQ-Learning~\cite{Yang+2019}, and SFOLS~\cite{Alegre+2022} (which uses OLS~\cite{Roijers2016,Mossalam+2016} for selecting which weight vectors to train on). 

In Fig.~\ref{fig:dst}, we evaluate each algorithm's expected utility (EU) and maximum utility (MUL) both \textit{(i)} if each algorithm---when optimizing a given preference---follows its best known policy (two leftmost plots); and \textit{(ii)} if each algorithm---when optimizing a given preference---follows the GPI policy derived from the known policies (two rightmost plots). 
We consider a setting where agents are constrained and have a limited budget of $4{,}000$ learning time steps per iteration. We now make a few important empirical observations, which hold for both above-mentioned settings (\textit{i} and \textit{ii}):

\noindent$\,\,\bullet$ The performance of the Envelope algorithm (both in terms of EU and MUL) has high variance since it trains on randomly-sampled weight vectors at each iteration. Furthermore, its asymptotic performance is always worse than our algorithms' performances. 

 \noindent$\,\,\bullet$ SFOLS converges to a suboptimal CCS: it consistently adds suboptimal policies (the best ones it could discover under the training time constraints) to its CCS. As a result, SFOLS incorrectly prioritizes 
 preferences.

 \noindent$\,\,\bullet$ GPI-LS, by contrast, 
 is consistently capable of identifying the most promising weight vectors, 
 i.e., those with the highest guaranteed improvement according to Eq.~\eqref{eq:best-w}. This allows it to always rapidly reach near-zero maximum utility loss. These results are consistent with Theorems~\ref{th:guaranteeCCS}~and~\ref{th:guaranteeEpsCCS}, which formally ensure that GPI-LS always converges to optimal solutions and that it monotonically improves the quality of its solutions over time (up to approximation errors due to experimental statistics having been computed based on finite data points). 
When GPI-LS also employs experience prioritization (that is, when combined with GPI-PD), the resulting method produces even higher improvements at each iteration, thereby enhancing overall sample efficiency.

\vspace{0.1cm}
\noindent \textbf{Minecart.}
\begin{figure}[h]
\centering

\begin{subfigure}[h]{\columnwidth}
\centering
\includegraphics[width=0.6\textwidth]{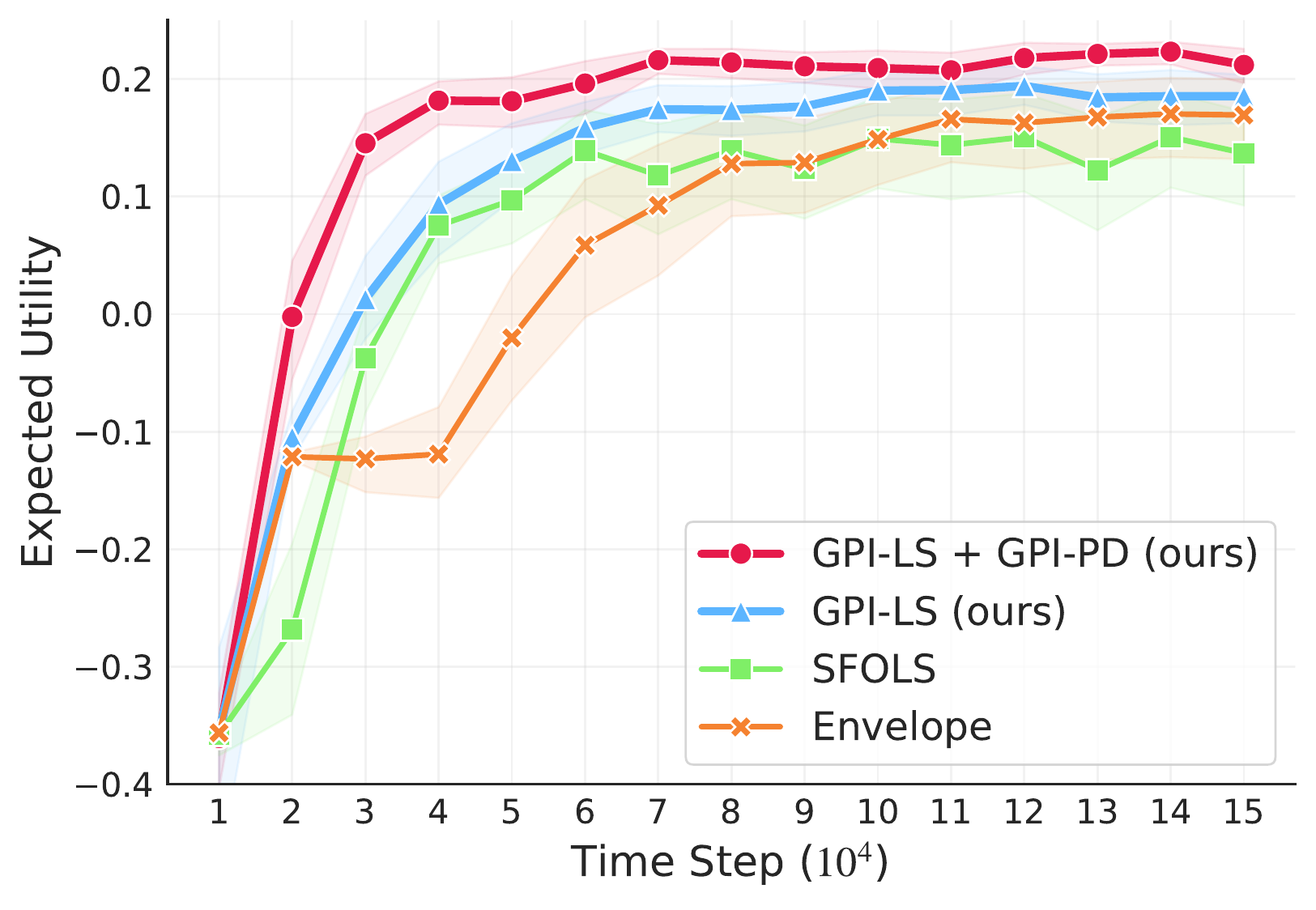}
\Description{Expected Utility over 100 weight vectors from $\Wspace$.}
\caption{Expected Utility over 100 weight vectors from $\Wspace$.}
\label{fig:minecart-utility}
\end{subfigure}

\begin{subfigure}[h]{\columnwidth}
\centering
\includegraphics[width=0.6\textwidth]{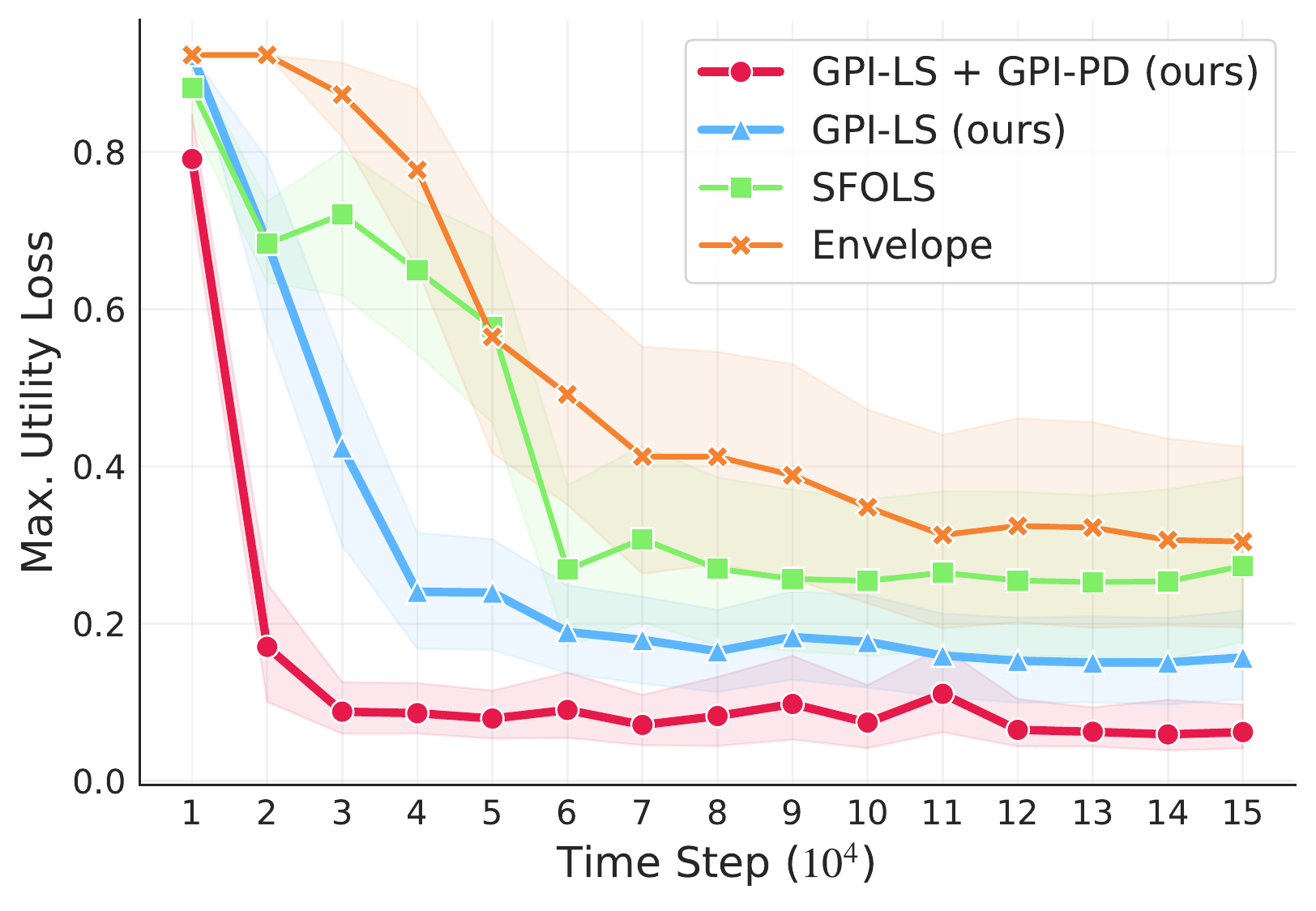}
\Description{Maximum Utility Loss.}
\caption{Maximum Utility Loss.}
\label{fig:minecart-mul}
\end{subfigure}

\Description{Evaluation of the algorithms on the Minecart domain.}
\caption{[Minecart Domain] Performance evaluation.}
\label{fig:minecart}

\end{figure}
This domain involves controlling a cart (Fig.~\ref{fig:envs}b) tasked with reaching ore mines, collecting and selling ores, and minimizing fuel consumption~\cite{Abels+2019}. There are three conflicting objectives, representing the agent's preferences for collecting each of the two types of ores, and for saving fuel. This is a continuous state space problem, and so we can directly deploy GPI-PD (Alg.~\ref{alg:gpi-pd}); GPI-PD uses GPI both to prioritize weight vectors and experiences (Section~\ref{sec:model}).
When comparing it with SFOLS and Envelope, we ensure that such competitors use the same neural network architecture and hyperparameters as GPI-PD, to allow for a fair comparison.

In Fig.~\ref{fig:minecart}, the performance of all algorithms follows a qualitatively similar behavior as observed in the Deep Sea Treasure experiment. Our methods (GPI-LS, and GPI-LS combined with GPI-PD) consistently identify optimal solutions, reach near-zero maximum utility loss, and achieve performance metrics that strictly dominate that of competitors. The Envelope algorithm is the least sample-efficient method during the first ten iterations---although it slightly surpasses the expected utility of SFOLS at the end of the learning process. These results further emphasize the importance of our weight prioritization technique (Section~\ref{sec:gpi-ls}), which allows GPI-LS to allocate its training time to preferences whose performances are guaranteed to be improvable.
As a result, GPI-LS outperforms the competing algorithms in terms of sample efficiency, and always converges to solutions with higher EU and MUL. Similarly, GPI-PD consistently reaches near-zero maximum utility loss.

\vspace{0.1cm}
\noindent \textbf{MO-Hopper.} 
\begin{figure}[h]
\centering

\begin{subfigure}[h]{\columnwidth}
\centering
\includegraphics[width=0.6\textwidth]{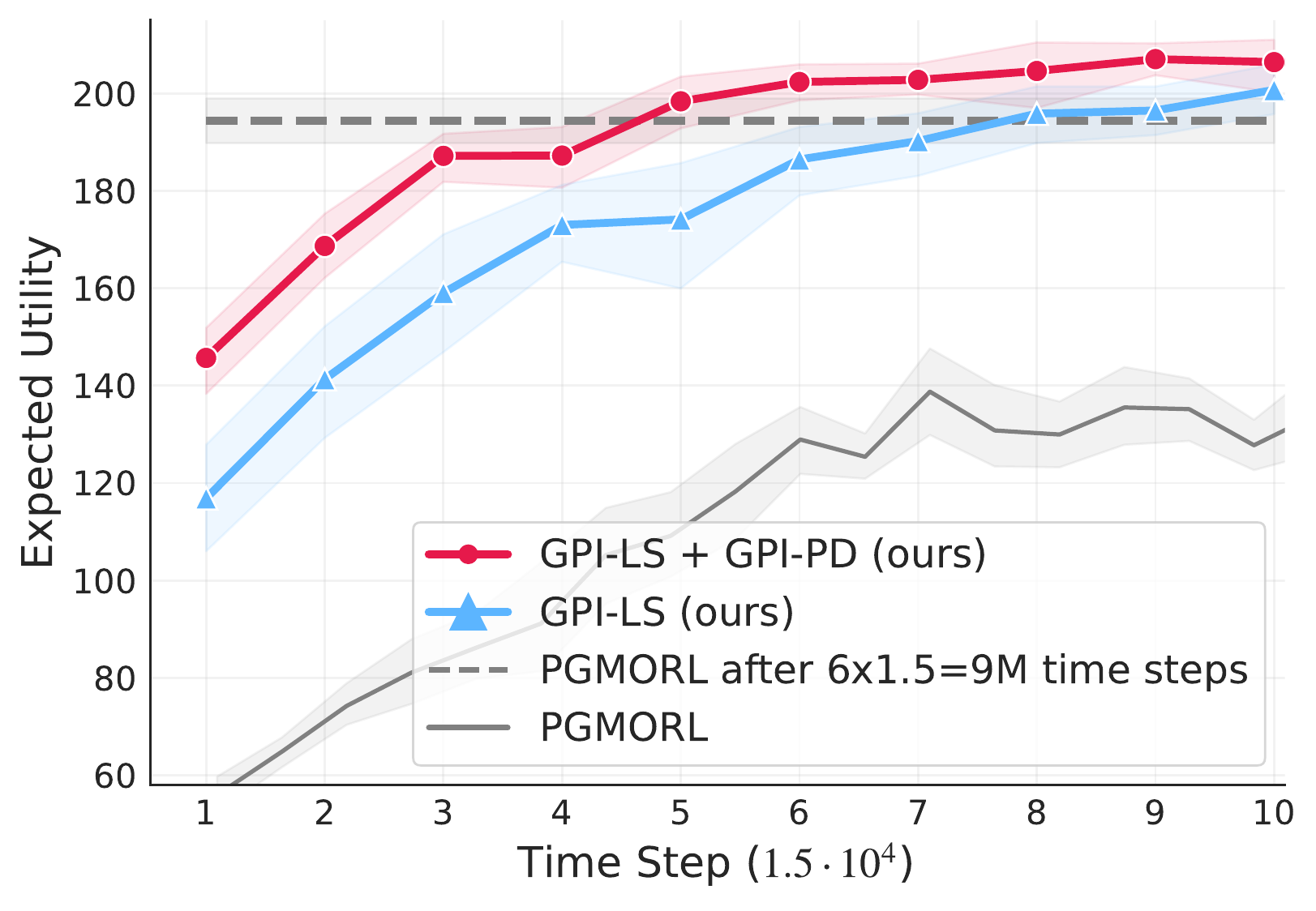}
\Description{.}
\caption{Expected Utility over 100 weight vectors from $\Wspace$.}
\label{fig:hopper-utility}
\end{subfigure}

\begin{subfigure}[h]{\columnwidth}
\centering
\includegraphics[width=0.6\textwidth]{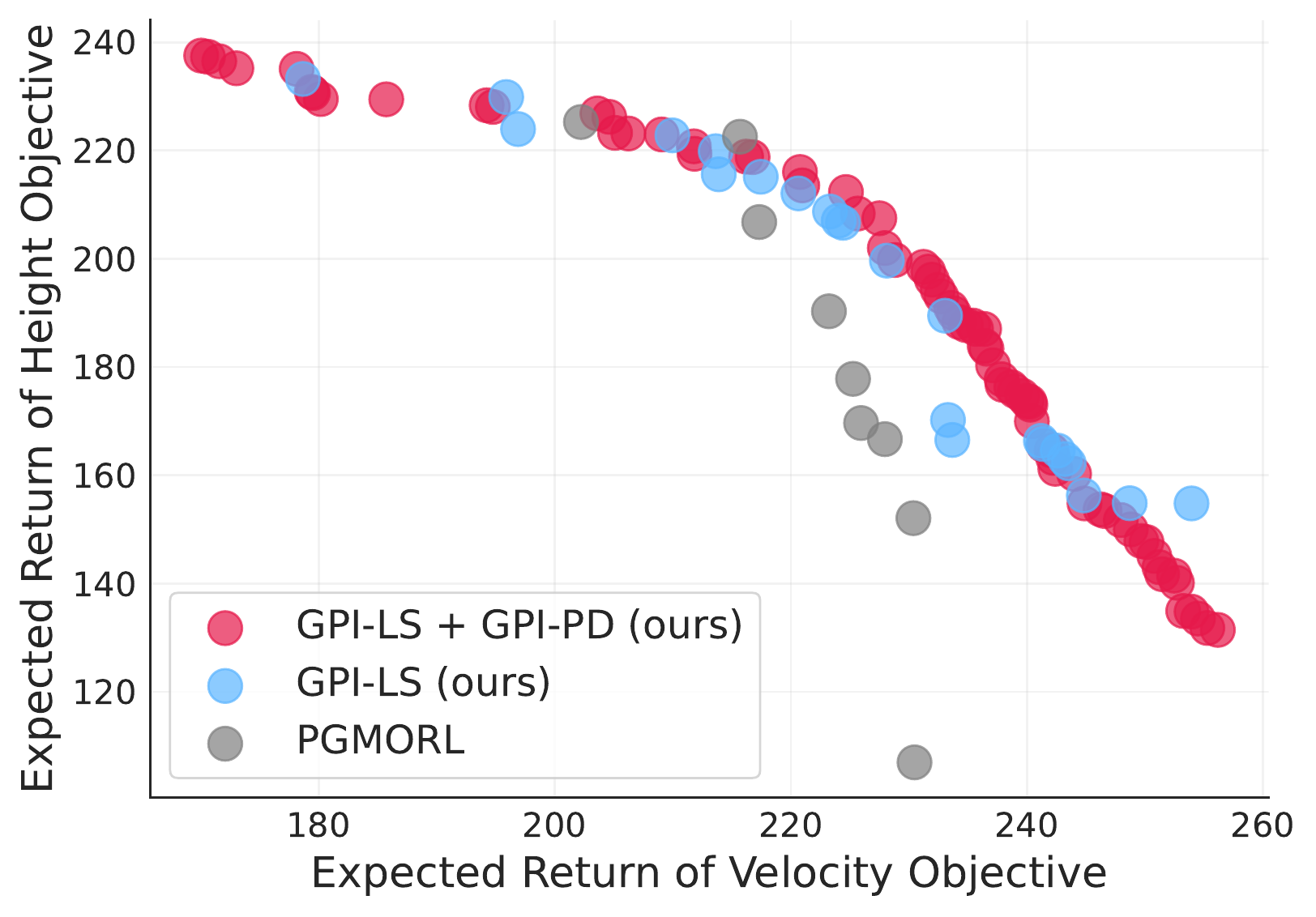}
\Description{Pareto frontier encountered by each method.}
\caption{Pareto frontier produced by each method.}
\label{fig:hopper-frontier}
\end{subfigure}

\Description{Evaluation of the algorithms on the MO-Hopper domain.}
\caption{[MO-Hopper]: Performance evaluation.}
\label{fig:hopper}

\end{figure}
This domain involves controlling a one-legged agent tasked with learning to hop while balancing two objectives: maximizing forward speed and jumping height. Unlike the previously-discussed domains, MO-Hopper operates over a continuous-action space. Neither Envelope nor SFOLS can tackle such a setting. For this reason, we compared our algorithms with a state-of-the-art MORL technique capable of dealing with continuous actions: PGMORL~\cite{Xu+2020icml}.
This is an evolutionary algorithm that assigns weight vectors to a population of agents (six, in our experiments, as in \cite{Xu+2020icml}). Agents are trained in parallel using PPO~\cite{Schulman+2017}. We extended our methods---both GPI-PD and GPI-LS---to deal with continuous actions by adapting the TD3 algorithm~\cite{Fujimoto+2018} to the MORL setting (see the Appendix).
In Fig.~\ref{fig:hopper}, we compare the expected utility and Pareto frontiers generated by GPI-PD and PGMORL. Here, we analyze Pareto frontiers because we cannot compute the maximum utility loss exactly, since the optimal CCS is unknown. Pareto frontiers were computed by selecting all Pareto-efficient policies among those collected during 20 runs of the algorithms, and evaluating each such policy 5 times over 100 weight vectors.
Once again, GPI-LS and GPI-PD have superior performance both in terms of expected utility and quality of the Pareto frontier. This is the case even though we allowed each of the six PGMORL agents to collect $1.5$ million experiences: they could interact with their environment \textit{ten times more often} than our method/agent. Even then, GPI-PD (with or without using a learned model) consistently achieved higher expected utility (Fig.~\ref{fig:hopper-utility}) during learning, and converged to a final solution with superior performance. Fig.~\ref{fig:hopper-frontier} depicts how the Pareto frontier identified by GPI-LS covers a broader range of behaviors that trade-off between the two objectives. This is made clear by observing that the points in the blue curve of Fig.~\ref{fig:hopper-frontier} dominate the corresponding Pareto frontier of PGMORL. Similarly, GPI-LS, when combined with GPI-PD, also produces a Pareto frontier that dominates that of PGMORL with the exception of one value vector. Notice that because GPI-LS+GPI-PD prioritize \textit{both} weight and experience selection, the resulting frontier is denser---it more thoroughly covers the space of all possible trade-offs between objectives.

\balance

\section{Related Work}

In this section, we briefly discuss some of the most relevant related work. For a complete 
discussion, see the Appendix. 

\noindent \textbf{Model-Free MORL.}
Several \textit{model-free} MORL algorithms have been proposed in the literature ~\cite{vanMoffaert&Nowe2014,Parisi+2017,Abdolmaleki+2020,Hayes+2022}. Here, we discuss the ones that are most related to our work.
In \cite{Abels+2019}, Abels et al. proposed training neural networks capable of predicting multi-objective action-value functions when conditioned on state and weight vector. Yang et al.~\cite{Yang+2019} extended the previously-mentioned approach to increase sample efficiency by introducing a novel operator for updating the action-value neural network parameters.
Our method differs with respect to these in that \textit{(i)} we introduced a novel, principled prioritization scheme to actively select weights on which the action-value network should be trained, in order to more rapidly learn a CCS (Section~\ref{sec:gpi-ls}); and \textit{(ii)} we introduced a principled prioritization method for sampling previous experiences for use in Dyna planning to accelerate learning optimal policies for particular agent preferences. The methods above, by contrast, use uniform sampling strategies that implicitly assume that all training data is equally important to rapidly identify optimal solutions.
Alegre et al.~\cite{Alegre+2022} showed how to use GPI in MORL settings, specifically in case the OLS algorithm is used to learn a CCS.
Importantly, however, the heuristic used by OLS to select which weight vectors to train is based on upper bounds that are frequently exceedingly optimistic and loose. This means that OLS often focuses its training efforts on optimizing policies that do not necessarily improve the maximum utility loss incurred by the resulting CCS. Our method, by contrast, uses a technique for selecting  which preferences to train on via a mathematically principled technique for determining a lower bound on performance improvements that are formally guaranteed to be achievable. This results in a novel active-learning approach that more accurately infers the ways by which a CCS can be rapidly improved.
Xu et al.~\cite{Xu+2020icml} proposed an evolutionary algorithm that trains a population of agents specialized in different weight vectors. We used a single neural network conditioned on the weight vector to simultaneously model \textit{all} policies in the CCS. Fig.~\ref{fig:hopper-utility} shows that our approach surpasses the performance of this algorithm, even when given ten times fewer environment interactions.

\noindent \textbf{Model-Based MORL.}
Compared to model-free MORL algorithms, model-based MORL methods have been relatively less explored. Wiering et al.~\cite{Wiering+2014} introduced a tabular multi-objective dynamic programming algorithm. %
Yamaguchi et al.~\cite{Yamaguchi+2019} proposed an algorithm that learns models for predicting expected multi-objective reward vectors, rather than multi-objective Q-functions. Importantly, they tackle the average reward setting, rather than the cumulative discounted return setting. \cite{Yamaguchi+2019} is limited to discrete state spaces and focuses its experiments on small MDPs with at most ten states. Similarly, \cite{Wiering+2014} assumes both discrete and deterministic MOMDPs.
Similarly, the model-based method 
in \cite{Agarwal+2022} can tackle tabular, discrete 
settings---even though it supports non-linear utility functions. 
The methods proposed in \cite{Wang&Sebag2012,Perez+2015,Painter+2020,Hayes+2021aamas} investigate multi-objective Monte Carlo tree search approaches, but unlike our algorithm, require prior access to known transition models of the MOMDP. 

\section{Conclusions}

In this paper we introduced two principled prioritization methods that significantly improve sample efficiency, as well as the first model-based MORL algorithm capable of dealing with continuous state spaces.
Both prioritization schemes we introduced were derived from properties of GPI and resulted in an effective, sample-efficient algorithm with important theoretical guarantees. In particular, we \textit{(i)} proved strong convergence guarantees to an optimal solution in a finite number of steps, or to an $\epsilon$-optimal solution (for a bounded $\epsilon$) if the agent is limited and can only identify possibly sub-optimal policies; \textit{(ii)} proved that our method is an anytime algorithm capable of monotonically improving the quality of the solution throughout the learning process; and \textit{(iii)} formally defined a bound that characterizes the maximum utility loss incurred by partial solutions computed by our technique at any given iteration. Finally, we empirically showed that our algorithm outperforms state-of-the-art MORL algorithms in qualitatively different multi-objective problems, both with discrete and continuous state and action spaces.
In future work, we would like to combine our algorithmic contributions with other types of model-based approaches, such as predecessor/backward models~\cite{Chelu+2020} and value equivalent models~\cite{Grimm+2021}. We are also interested in extending our method to settings with non-linear utility functions.





\begin{acks}
This study was financed in part by the following Brazilian agencies: Coordenação de Aperfeiçoamento de Pessoal de Nível Superior - Brazil (CAPES) - Finance Code 001; CNPq (grants 140500/2021-9 and 304932/2021-3); and FAPESP/MCTI/CGI (grant 2020/05165-1).
This research was partially supported by funding from the Flemish Government under the ``Onderzoeksprogramma Artifici\"{e}le Intelligentie (AI) Vlaanderen'' program and the Research Foundation Flanders (FWO) [G062819N].
\end{acks}



\bibliographystyle{ACM-Reference-Format} 
\bibliography{AL,MZ,OURS,stringDefs}


\end{document}



\pagestyle{fancy}
\fancyhead{}


\maketitle

\onecolumn
\begin{center}
    \huge
    \textbf{Appendix: Sample-Efficient Multi-Objective Learning via\\ Generalized Policy Improvement Prioritization}
        
    \vspace{0.4cm}
        
\end{center}

\section{Proof of Theorem 3.7}

\paragraph{Theorem 3.7}
Let $\Pi$ be an arbitrary set of policies, and let $\pi_{\w}$ in $\Pi$ be a deterministic policy tasked with optimizing some $\w\in\Wspace$. Then, $q^{\gpi}_{\w}(s,a) = q^{\pi}_{\w}(s,a)$ \textbf{\textit{for all}}
state-action pairs in $\Sspace {\times} \Aspace$ if and only if  $q^*_{\w}(s,a) = q^{\pi}_{\w}(s,a)$. In other words, $\pi_{\w}$ is guaranteed to be an optimal policy for $\w$ iff the GPI policy computed over $\Pi$, for optimizing $\w$, cannot improve the $q$-function of $\pi_{\w}$ for any state-action pairs.
\begin{proof}
We will prove both directions of the equivalence separately. First, we have to show that:
\begin{eqnarray}
  q^{\gpi}_{\w}(s,a) = q^{\pi}_{\w}(s,a), \ \forall (s,a)\in\Sspace{\times}\Aspace \implies q^*_{\w}(s,a) = q^{\pi}_{\w}(s,a), \ \forall (s,a)\in\Sspace{\times}\Aspace.
\end{eqnarray}

First, notice that if $q^{\gpi}_{\w}(s,a) = q^{\pi}_{\w}(s,a)$, then:
\begin{align}
    q^{\pi}_{\w}(s,a) &= q^{\gpi}_{\w}(s,a) \\
                      \label{eq:bla2}
                      &\geq \max_{\pi'\in\Pi} q^{\pi'}_{\w}(s,a). \ \text{(due to Theorem 1 of Barreto et al.~\cite{Barreto+2017})}.
\end{align}
This implies that:
\begin{align}
    \pi^\gpi(s;\w) &\in \argmax_{a\in\Aspace}\max_{\pi'\in\Pi} q^{\pi'}_{\w}(s,a) \\
    \label{eq:bla1}
                  &\in \argmax_{a\in\Aspace} q^{\pi}_{\w}(s,a) \ \text{(because of Eq.~\eqref{eq:bla2} and since $\pi\,\in\,\Pi$)}.
\end{align}
We can then show that:
\begin{align}
    q^{\gpi}_{\w}(s,a) &= \Ex[\vect{R}_t \cdot \w + \gamma q^{\gpi}_{\w}(S_{t+1}, \pigpi(S_{t+1};\w)) \ | \ S_t = s, A_t = a ]\\
     &= \Ex[\vect{R}_t \cdot \w + \gamma q^{\pi}_{\w}(S_{t+1}, \pigpi(S_{t+1};\w)) \ | \ S_t = s, A_t = a ]\\
     &= \Ex[\vect{R}_t \cdot \w + \gamma q^{\pi}_{\w}(S_{t+1}, \argmax_{a'\in\Aspace} q^{\pi}_{\w}(s,a')) \ | \ S_t = s, A_t = a ] \ \text{(due to the result shown in Eq.~\eqref{eq:bla1})}\\
     &= \Ex[\vect{R}_t \cdot \w + \gamma \max_{a'\in\Aspace} q^{\pi}_{\w}(S_{t+1}, a') \ | \ S_t = s, A_t = a ].
\end{align}
Because $q^{\gpi}_{\w}(s,a) = q^{\pi}_{\w}(s,a)$, then:
\begin{align}
    \label{eq:bellman-optimal}
    q^{\pi}_{\w}(s,a) &= \Ex[\vect{R}_t \cdot \w + \gamma \max_{a'\in\Aspace} q^{\pi}_{\w}(S_{t+1}, a') \ | \ S_t = s, A_t = a ].
\end{align}
Notice that Eq.~\eqref{eq:bellman-optimal} is the Bellman optimality equation for action-value functions.
Hence, we can conclude that $q^{\pi}_{\w}(s,a) = q^*_{\w}(s,a)$.

For the opposite direction of the proof, we have to show that:
\begin{eqnarray}
    \label{eq:direction2_thm3.7}
    q^*_{\w}(s,a) = q^{\pi}_{\w}(s,a), \ \forall (s,a)\in\Sspace{\times}\Aspace \implies  q^{\gpi}_{\w}(s,a) = q^{\pi}_{\w}(s,a), \ \forall (s,a)\in\Sspace{\times}\Aspace.
\end{eqnarray}
This can be shown by the following steps:
\begin{align}
   q^{\gpi}_{\w}(s,a) &\geq \max_{\pi'\in\Pi} q^{\pi'}_{\w}(s,a) \ \text{(Theorem 1 of Barreto et al.~\cite{Barreto+2017})} \\
                      &\geq q^{\pi}_{\w}(s,a)  \ (\text{since } \pi\in\Pi) \\
                      &= q^*_{\w}(s,a) \ (\text{since the antecedent (LHS) of Eq.~\eqref{eq:direction2_thm3.7} states that } q^{\gpi}_{\w}(s,a) = q^{*}_{\w}(s,a)).
\end{align}
From the definition of the optimal action-value it is not possible that $q^{\gpi}_{\w}(s,a) > q^*_{\w}(s,a)$. Thus, it follows that $q^{\gpi}_{\w}(s,a) = q^*_{\w}(s,a)$.
\end{proof}

\section{Experiments Details}

\subsection{Environments}

All the environments used in the experiments are available in the MO-Gym library~\cite{Alegre+2022bnaic}.

\paragraph{Deep Sea Treasure}
The Deep Sea Treasure is a classic MORL environment \cite{Vamplew+2011,vanMoffaert&Nowe2014,Abels+2019,Yang+2019}.
The agent's state at a given time step $t$ is its coordinates in the grid, $S_t = [x,y]$.
The action space consists of four directions the agent can move to, $\mathcal{A} = \{ \text{up},\text{down},\text{left},\text{right}\}$.
The first component of the reward vector $\vect{r}(s,a,s') \in \mathbb{R}^2$ is the treasure value\footnote{We adopted the treasures values as defined in \cite{Yang+2019}.} (or zero if the agent is in a blank cell), and the second component is a time penalty of $-1$ in all states. 
The cells with treasures are also terminal states.
We considered a discount factor of $\gamma = 0.99$ in this domain.
There are ten different optimal values in this domain's CCS, each corresponding to a policy that reaches one of the ten treasures on the map.

\paragraph{Minecart}
In this domain, introduced in \cite{Abels+2019}, the agent controls a cart that starts in a location in the upper-left corner of the map, and must go to one of the existing ore mines and collect two different types of ores. Then, the cart must return to the base to sell the ores.
The state space $\Sspace\in\mathbb{R}^{7}$ consists of the cart's 2D coordinates, its speed, the sine and cosine of the angle of the cart orientation, and the current amount of collected ore, normalized by the maximum capacity of ore that the cart can store.
The action space contains the following six actions: $\Aspace = \{\mathrm{mine, left, right, accelerate, brake, none}\}$.
The first two components of the reward function $\vect{r}(s,a,s') \in \mathbb{R}^3$ are the amount of ore sold for each of the two types of ores (notice that this is a sparse reward, given to the agent only when it returns to the base). The last component is the amount of fuel used when the agent selects the actions mine or accelerate, plus a constant fuel cost for all time steps.
We used $\gamma = 0.98$ in this domain, as in \cite{Abels+2019}.

\paragraph{MO-Hopper}
This environment extends the \enquote{Hopper-v4} environment of the Gym library~\cite{Brockman+2016}. 
The state space $\Sspace \in \mathbb{R}^{11}$ consists of positional values of different body parts of the hopper, followed by the velocities of those individual parts.
The action space $\Aspace\in\mathbb{R}^3$ represents the torques applied to each of the three joints of the robot.
The multi-objective reward function $\vect{r}(s,a,s') \in \mathbb{R}^2$ is defined similarly as in \cite{Xu+2020icml}:
\begin{align}
     r_1(s,a,s') &= v_x + C, \\
    r_2(s,a,s') &= 10 (h - h_{\mathrm{init}}) + C,
\end{align}
where $v_x$ is the agent's horizontal velocity, $h$ and $h_{\mathrm{init}}$ are the agent's current height and initial height, respectively, and $C = 1 - \sum_{i} a_{i}^2$ is a term containing a bonus for the agent's being alive as well as an energy cost.
We considered $\gamma=0.99$.

\subsection{Experiments Details}

The code containing the proposed methods' implementation and the necessary scripts to reproduce our results are available at MORL-Baselines library: \url{https://github.com/LucasAlegre/morl-baselines}.

\subsubsection{Weight Vectors}
To compute the corner weights ($\mathrm{CornerWeights}(\Vset)$ in GPI-PD) we used pycddlib (\url{https://github.com/mcmtroffaes/pycddlib/}) to efficiently enumerate the vertices of the polyhedron $P$ (Definition 3.1 in the paper).
To generate equidistant weight vectors covering the simplex $\Wspace$, we used the Riesz s-Energy method~\cite{Blank+2021} implemented in pymoo~\cite{Blank&Deb2020}.

\subsubsection{Neural Network Architectures}
To further increase the sample efficiency of our algorithm, we employed a recently-proposed technique of using Dropout and Layer Normalization \cite{Hiraoka+2022} in the neural networks we used to approximate the action-value function $\vect{Q}_{\theta}(s,a,\w)$.
This allows us to use a higher update-to-data (UTD) ratio, that is, a higher number of gradient updates (we used $G = 20$) per each time step in the real environment.
In the Minecart domain, each neural network was modeled as a multilayer perceptron (MLP) consisting of 4 layers with 256 neurons. In the MO-Hopper domain, we used 2 layers with 256 neurons for both the critic and actor networks.
We used the same neural network architectures both in our method and when implementing competitor algorithms (Envelope~\cite{Yang+2019} and SFOLS~\cite{Alegre+2022}), in order to make the experimental comparisons fair.
We used Adam \cite{Kingma&Ba2015} (with learning rate $3\cdot10^{-4}$ and mini-batches of size 256) as the first-order gradient optimizer used to train all neural networks that were used in our algorithms and competing methods.

\subsubsection{Dyna and Learned Models}

We used an ensemble of $n=5$ probabilistic neural networks in both the Minecart and MO-Hopper environments to learn the dynamics model $\model_{\varphi}(S_{t+1},\vect{R_t}|S_t,A_t)$.
For the Minecart domain, each network was an MLP with 3 layers with 256 neurons each, and for the MO-Hopper domain, we used an MLP with 4 layers with 200 neurons each.
The probabilistic neural networks used for the dynamics model were trained with early stopping using a holdout validation subset of the experiences in the buffer $\buffer$, as is commonly done when training these models~\cite{Chua+2018,Janner+2019}.

In the Deep Sea Treasure, we used $H=5$ Dyna steps for each step in the real environment. Notice that all algorithms used Dyna with uniform sampling and the same number of updates. GPI-PD's only difference from GPI-LS is that it uses our proposed GPI-based prioritization for sampling $H$ states based on which it performs Dyna updates.
Notice that for the tabular case, there are no mini-batches and we use a slightly different version of GPI-PD designed for the tabular case (see Section~\ref{sec:tabular} of this Appendix). 
For the Minecart domain, the model $\model_{\varphi}$ was updated every 250 time steps and then used to generate 25,000 simulated experiences. For the MO-Hopper domain, the model was updated every 250 time steps and then used to generate 50,000 simulated experiences.
This corresponds to $H=25000/250 = 100$ and $H=50000/250 = 200$, respectively.
To construct each mini-batch (line 21 of Algorithm 2 in the paper), we used a ratio of simulated-to-real experiences of $\beta=0.5$ for Minecart, and $\beta=0.9$ for MO-Hopper.

To sample from a replay buffer based on the priorities computed by our proposed technique, GPI-PD, we employed the commonly-used sum tree data structure~\cite{Schaul+2016}. We follow Fujimoto et al.~\cite{Fujimoto+2020} and computed the likelihood of sampling the $i$-th experience in the buffer $\buffer$ as follows:
\begin{equation}
    P_{\w}(i) = \frac{\max (|\delta_i|^\alpha, \kappa)}{\sum_j \max (|\delta_j|^\alpha, \kappa)},
\end{equation}
where $\delta_i$ is the priority of the $i$-th experience, $\alpha$ is a parameter to smooth out extreme values, and $\kappa$ is a minimum allowed priority. 
We used $\alpha = 0.6$ in all domains, $\kappa = 0.001$ for the Deep Sea Treasure domain, $\kappa = 0.01$ for the Minecart domain, and $\kappa = 0.1$ for the MO-Hopper domain.

\subsubsection{Exploration} We used $\epsilon$-greedy exploration in the Deep Sea Treasure and Minecart domains. In Deep Sea Treasure, we linearly annealed $\epsilon$ from $1$ to $0$ during the first 50,000 time steps. In the Minecart domain, we linearly annealed $\epsilon$ from $1$ to $0.05$ during the first 50,000 time steps.
In the MO-Hopper domain, we followed the original TD3 algorithm and added a zero-mean Gaussian noise with standard deviation of $0.02$ to the actions outputted by the actor network.

\subsection{Tabular GPI-PD}
\label{sec:tabular}

In Algorithm~\ref{alg:tabular-gpi-pd}, below, we introduce a tabular version of the GPI-PD algorithm presented in the paper. This algorithm was used in the Deep Sea Treasure experiments. Here, we used a learning rate of $\alpha=0.3$.

\begin{algorithm}[h!]
\caption{Tabular \textbf{GPI} \textbf{P}rioritized \textbf{D}yna (\textbf{GPI-PD})}
\label{alg:tabular-gpi-pd}

\DontPrintSemicolon
\SetKwInOut{Input}{Input}
\SetKwProg{everyn}{}{ do}{}

\Input{MOMDP M, timesteps per iteration $N$, learning rate $\alpha$}

Initialize tabular model $\model$, buffer $\buffer$\;
$\Vset \leftarrow \{\}, \Pi \leftarrow \{\}$\;
$\w \leftarrow [1,0,...,0]$\;
\For{$t = 0 ... \infty$}{
\everyn(\Comment*[f]{GPI Linear Support}){Every $N$ time steps}{
    Add $\vect{v}^{\pi_{\w}}$ to $\Vset$\;
    $\Pi, \Vset \leftarrow \mathrm{RemoveDominated}(\Pi,\Vset)$\;
    $\Wspace_{\mathrm{corner}} \leftarrow \mathrm{CornerWeights(\Vset)}$\;
    $\w \leftarrow \argmax_{\w\in\Wspace_{\mathrm{corner}}}
   (v^{\gpi}_{\w} - \max_{\pi\in\Pi} v^{\pi}_{\w})$\;
    Initialize new policy $\pi_{\w}$ with $\vect{Q}^{\pi_{\w}}(s,a) \leftarrow \vect{Q}^{\pi'}(s,a), \forall (s,a)\in\Sspace{\times}\Aspace$ , where $\pi' = \argmax_{\pi\in\Pi} \vect{v}^{\pi} \cdot \w$\;
    Add $\pi_{\w}$ to $\Pi$\;
}
\If{$S_t$ is terminal}{
    $S_t \leftarrow$ sample state from initial state distribution $\mu$\;
}

$A_t \leftarrow \pigpi(S_t; \w)$  \ \Comment*[r]{Follow GPI policy} 
Execute $A_t$, observe $S_{t+1}$ and $\vect{R}_t$\;

$A' \leftarrow  \pigpi(S_{t+1}; \w)$\;
$\vect{Q}^{\pi_{\w}}(S_t,A_t) \leftarrow \vect{Q}^{\pi_{\w}}(S_t,A_t) + \alpha (\vect{R}_t + \gamma\vect{Q}^{\pi_{\w}}(S_{t+1},A') - \vect{Q}^{\pi_{\w}}(S_t,A_t))$\;

\For{$\pi_{\w'} \in \Pi$}{
$A' \leftarrow  \pigpi(S_{t+1}; \w')$\;
$\vect{Q}^{\pi_{\w'}}(S_t,A_t) \leftarrow \vect{Q}^{\pi_{\w'}}(S_t,A_t) + \alpha (\vect{R}_t + \gamma\vect{Q}^{\pi_{\w'}}(S_{t+1},A') - \vect{Q}^{\pi_{\w'}}(S_t,A_t))$\;
}

Add $(S_t,A_t,\vect{R}_t,S_{t+1})$ to $\buffer$ with priority $P_{\w}(S_t,A_t)$\;

Update the tabular model $\model$ with the experience  $(S_t,A_t,\vect{R}_t,S_{t+1})$\;

\For(\Comment*[f]{GPI-Prioritized Dyna}){$H$ Dyna planning steps}{
    Sample $(S, A) \sim \buffer$ according to $P_{\w}$ \;
    
    $(S', \vect{R}) \sim \model(\cdot|S,A)$\;
    
    $A' \leftarrow  \pigpi(S'; \w)$\;
    $\vect{Q}^{\pi_{\w}}(S,A) \leftarrow \vect{Q}^{\pi_{\w}}(S,A) + \alpha (\vect{R} + \gamma\vect{Q}^{\pi_{\w}}(S',A') - \vect{Q}^{\pi_{\w}}(S,A))$\;
    
    Update priority $P_{\w}(S,A)$\;
    \For{$\pi_{\w'} \in \Pi$}{
    $A' \leftarrow  \pigpi(S'; \w')$\;
    $\vect{Q}^{\pi_{\w'}}(S,A) \leftarrow \vect{Q}^{\pi_{\w'}}(S,A) + \alpha (\vect{R}_t + \gamma\vect{Q}^{\pi_{\w'}}(S',A') - \vect{Q}^{\pi_{\w'}}(S,A))$\;
    }

}

}

\end{algorithm}

\subsection{Continuous-Action GPI-PD}

In order to allow our algorithms to tackle problems with continuous action spaces, we extended TD3~\cite{Fujimoto+2018} to the MORL setting; we call this variant MOTD3.
In particular, we extended TD3 critic networks to receive as input the weight vector $\w$ and output multi-objective values: $\vect{Q}_{\theta}(s,a,\w) : \Sspace {\times} \Aspace {\times} \Wspace \mapsto \mathbb{R}^m$. The policy network, $\pi_{\phi}(s,\w)$, was also conditioned on the weight vector.
The MOTD3 loss function of each of the two critic network, with respect to their corresponding parameters $\{\theta_i\}_{i\in\{1,2\}}$, for a given transition $(s,a,\vect{r},s')$ and weight vector $\w$, is
\begin{equation}
    \mathcal{L}(\theta;\w) =  (\vect{y} - \vect{Q}_{\theta}(s,a,\w))^2,
\end{equation}
where
\begin{equation}
    \vect{y} = \vect{r} + \gamma \vect{Q}_{\theta^{-}_{k}}(s',a',\w), \quad k = \argmin_{i\in\{1,2\}} \vect{Q}_{\theta^{-}_{i}}(s',a',\w) \cdot \w, \quad a' = \pi_{\phi}(s',\w).
\end{equation}
MOTD3's actor network parameters, $\phi$, are then updated following the deterministic policy gradient:
\begin{equation}
    \nabla J(\phi;\w) = \nabla_a \vect{Q}_{\theta}(s,a,\w) \cdot \w |_{a=\pi_{\phi}(s,\w)} \nabla_{\phi}\pi_{\phi}(s,\w). 
\end{equation}

The challenge of extending GPI to the continuous action space is due to the fact that GPI requires computing a maximum over the action space $\Aspace$ in its definition  $\rightarrow$ (see Eq.(4) in the paper). Because this is infeasible in the case $\Aspace$ is infinite, we compute GPI by considering only the finite set of actions returned by the policy network for the weight vectors in the weight support $\Wsupport$. Let $\hat{\Aspace} = \{a \ | \ a = \pi_{\phi}(s,\w) \ \text{for all} \ \w \in \Wsupport\}$ be the aforementioned finite set of actions. The GPI policy is then computed as:
    \begin{eqnarray}
\label{eq:fa-gpi-td3}
   \pigpi(s;\w) \in \argmax_{a\in\hat{\Aspace}}\max_{\w'\in\Wset} \vect{Q}_{\theta}(s,a,\w') \cdot \w .
\end{eqnarray}

\section{Related Work}

In this section, we briefly discuss some of the most relevant related work.

\paragraph{Model-Free MORL}
Several \textit{model-free} MORL algorithms have been proposed in the literature ~\cite{vanMoffaert&Nowe2014,Parisi+2017,Abdolmaleki+2020,Hayes+2022}. Here, we discuss the ones that are most related to our work.
In \cite{Abels+2019}, Abels et al. proposed training neural networks capable of predicting multi-objective action-value functions when conditioned on state and weight vector. Yang et al.~\cite{Yang+2019} extended the previously-mentioned approach to increase sample efficiency by introducing a novel operator for updating the action-value neural network parameters.
Our method differs with respect to these in that \textit{(i)} we introduced a novel, principled prioritization scheme to actively select weights on which the action-value network should be trained, in order to more rapidly learn a CCS (Section 3.2); and \textit{(ii)} we introduced a principled prioritization method for sampling previous experiences for use in Dyna planning to accelerate learning optimal policies for particular agent preferences. The methods above, by contrast, use uniform sampling strategies that implicitly assume that all training data is equally important to rapidly identify optimal solutions.
Alegre et al.~\cite{Alegre+2022} showed how to use GPI in MORL settings, specifically in case the OLS algorithm is used to learn a CCS.
Importantly, however, the heuristic used by OLS to select on which weight vectors to train is based on upper bounds that are frequently exceedingly optimistic and loose. This means that OLS often focuses its training efforts on optimizing policies that do not necessarily improve the maximum utility loss incurred by the resulting CCS. Our method, by contrast, uses a technique for selecting  which preferences to train on via a mathematically principled technique for determining a lower bound on performance improvements that are formally guaranteed to be achievable. This results in a novel active-learning approach that more accurately infers the ways by which a CCS can be rapidly improved.
Xu et al.~\cite{Xu+2020icml} proposed an evolutionary algorithm that trains a population of agents specialized in different weight vectors. We used a single neural network conditioned on the weight vector to simultaneously model \textit{all} policies in the CCS. Fig. 4a shows that our approach surpasses the performance of this algorithm, even when given ten times fewer environment interactions.

\paragraph{Model-Based MORL}
Compared to model-free MORL algorithms, model-based MORL methods have been relatively less explored. Wiering et al.~\cite{Wiering+2014} introduced a tabular multi-objective dynamic programming algorithm. %
Yamaguchi et al.~\cite{Yamaguchi+2019} proposed an algorithm that learns models for predicting expected multi-objective reward vectors, rather than multi-objective Q-functions. Importantly, they tackle the average reward setting, rather than the cumulative discounted return setting. \cite{Yamaguchi+2019} is limited to discrete state spaces and focuses its experiments on small MDPs with at most ten states. Similarly, \cite{Wiering+2014} assumes both discrete and deterministic MOMDPs.
Similarly, the model-based method 
in \cite{Agarwal+2022} can tackle tabular, discrete 
settings---even though it supports non-linear utility functions. 
The methods proposed in \cite{Wang&Sebag2012,Perez+2015,Painter+2020,Hayes+2021aamas} investigate multi-objective Monte Carlo tree search approaches, but unlike our algorithm, require prior access to known transition models of the MOMDP. 

\paragraph{Dyna Search-Control}
The case in which the agent performs Dyna without learning a model, and instead using only a experience replay buffer containing old experiences, has been extensively explored~\cite{Lin1992,Schaul+2016,Fujimoto+2020,Sinha+2022}.
Nonetheless, using a learned model has been shown to introduce many advantages and increase the sample-efficiency of deep RL algorithms~\cite{Janner+2019,Abbas+2020,Hafner+2021}. Among them is the capability of querying next states for on-policy actions, whereas experience replay can only replay (likely off-policy) samples previously stored.
Janner et al.~\cite{Janner+2019} combines Dyna with deep RL by learning an ensemble of probabilistic neural networks~\cite{Chua+2018}, and proposes using short model rollouts to avoid cumulative model prediction errors.
Recently, Pan et al.~\cite{Pan+2022} showed limitations of prioritized experience replay~\cite{Schaul+2016}, and proposed to perform Dyna planning using a technique that generates states with higher TD-error, and simultaneously shows a good coverage of the state space.
However, these techniques are for the standard RL setting where a single policy is learned for a single reward function. 
We introduced a technique that takes into account the policies that were already learned for different preferences over multiple reward functions.

\paragraph{Universal Successor Features}
Alegre et al.~\cite{Alegre+2022} showed that multi-objective action-value functions are mathematically equivalent to successor features (SFs)~\cite{Barreto+2017}, in case multi-objective reward vectors are interpreted as linear reward features. They showed that learning a CCS over SFs is equivalent to solving the problem of optimal policy transfer in the case in which the tasks are linear combinations of features.
Following the same principles, we can see conditioned neural networks $\vect{Q}_{\theta}(s,a,\w)$ as a universal successor feature approximator~\cite{Borsa+2019}.
Hence, the contributions of this paper may also be employed to increase the sample efficiency of SF-based algorithms.

\bibliographystyle{ACM-Reference-Format} 
\bibliography{AL,MZ,OURS,stringDefs}